\documentclass[10pt,twocolumn,letterpaper]{article}
\usepackage[utf8]{inputenc}
\usepackage[T1]{fontenc}
\usepackage{balance}
\usepackage[pagenumbers]{cvpr} % To force page numbers, e.g. for an arXiv version

\usepackage{amsmath,amsfonts,bm}

\def\eqref#1{equation~\ref{#1}}
\def\1{\bm{1}}

\DeclareMathAlphabet{\mathsfit}{\encodingdefault}{\sfdefault}{m}{sl}
\SetMathAlphabet{\mathsfit}{bold}{\encodingdefault}{\sfdefault}{bx}{n}

\usepackage{hyperref}
\usepackage{url}
\usepackage{multirow}

\usepackage{booktabs} % for professional tables
\usepackage{amsthm}
\usepackage{tikz}
\usepackage{graphicx}
\usepackage{readarray}
\usepackage{pgfplots}
\usepackage{pgfplotstable}
\pgfplotsset{compat=1.18}
\usepackage{bm}
\usetikzlibrary{shapes.misc, positioning, calc, arrows.meta,patterns,patterns.meta}
\usepackage{colortbl}
\usepackage{float}
\usepackage{listings}
\usepackage{subcaption}
\usepackage{enumitem}
\usepackage[table,dvipsnames]{xcolor}
\usepackage{makecell}
\usepackage{epigraph}
\setlist{nosep}

\definecolor{codegreen}{rgb}{0,0.6,0}
\definecolor{codegray}{rgb}{0.5,0.5,0.5}
\definecolor{codepurple}{rgb}{0.58,0,0.82}
\definecolor{backcolour}{rgb}{0.95,0.95,0.92}
\lstdefinestyle{mystyle}{
    backgroundcolor=\color{backcolour},   
    commentstyle=\color{codegreen},
    keywordstyle=\color{magenta},
    numberstyle=\tiny\color{codegray},
    stringstyle=\color{codepurple},
    basicstyle=\ttfamily\footnotesize,
    breakatwhitespace=false,         
    breaklines=true,                 
    captionpos=b,                    
    keepspaces=true,                 
    numbers=left,                    
    numbersep=5pt,                  
    showspaces=false,                
    showstringspaces=false,
    showtabs=false,                  
    tabsize=2
}

\newcommand{\greyrule}{\arrayrulecolor{black!30}\midrule\arrayrulecolor{black}}

\DeclareMathOperator{\PoM}{\text{PoM}}

\newcommand{\pomicon}{%
  \raisebox{-0.2ex}{\includegraphics[height=1em]{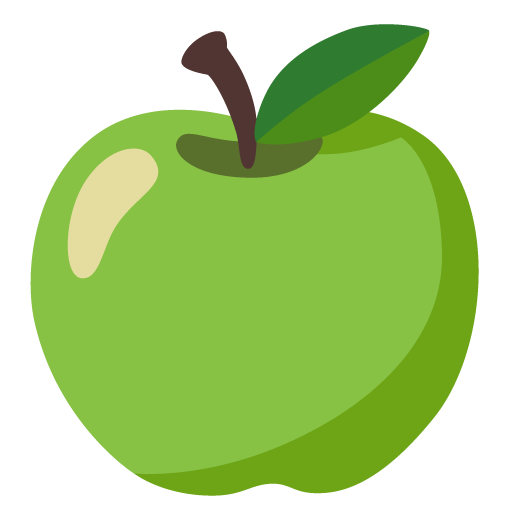}}%
}

\theoremstyle{plain}
\newtheorem{theorem}{Theorem}[section]
\newtheorem{proposition}[theorem]{Proposition}
\newtheorem{lemma}[theorem]{Lemma}

\theoremstyle{definition}

\theoremstyle{remark}

\definecolor{colPOM}{RGB}{102,200,0}      % crisp apple green
\definecolor{colTrans}{RGB}{70,110,220}     
\definecolor{colFlash}{RGB}{128,0,128}    % keep purple

\newcommand{\tokengrid}[4]{%
\begin{tikzpicture}[scale=0.5]
  \foreach \i in {0,...,\numexpr#2-1} {
    \foreach \j in {0,...,\numexpr#1-1} {
      \draw[fill=#4, rounded corners, very thick] 
        (\j,-\i *\ysep) rectangle ++(#3,1);
    }%
  }%
\end{tikzpicture}%
}%
\definecolor{cvprblue}{rgb}{0.21,0.49,0.74}

\usepackage[skip=6pt]{caption}
\title{\pomicon PoM: A Linear-Time Replacement for Attention with the Polynomial Mixer}

\author{David Picard$^1$, Nicolas Dufour$^{1,2}$, Lucas Degeorge$^{1,2,4}$, Arijit Ghosh$^1$, Davide Allegro$^3$, Tom Ravaud$^1$,\\ Yohann Perron$^{1,5}$, Corentin Sautier$^{1,6}$, Zeynep Sonat Baltaci$^1$, Fei Meng$^1$, Syrine Kalleli$^1$, Marta \\
López-Rauhut$^1$, Thibaut Loiseau$^1$, Ségolène Albouy$^1$, Raphael Baena$^1$, Elliot Vincent$^7$, Loic Landrieu$^1$\\
{\small $^1$LIGM, CNRS, Univ Gustave Eiffel, ENPC, Institut Polytechnique de Paris, France}\\
{\small $^2$LIX, École Polytechnique, CNRS, IP Paris, France}
\quad{\small $^3$Department of Information Engineering, Universit\`a degli Studi di Padova, Italy}\\
\small $^4$AMIAD, Pole recherche,\quad$^5$ EFEO$\quad^6$Valeo.ai, France,\quad
{\small $^7$LASTIG, Univ Gustave Eiffel, IGN, Géodata Paris, Paris, France}
}
\begin{document}
\twocolumn[{%
 \renewcommand\twocolumn[1][]{#1}
 \maketitle
 \vspace{-8mm}
    \centering 
     \noindent
\begin{minipage}[t]{0.33\textwidth}
\centering
\begin{tikzpicture}
\begin{semilogxaxis}[
    width=.9\linewidth,
    height=4.0cm,
    xlabel={\scriptsize Sequence length (tokens)},
    ylabel={\scriptsize Execution time (s/batch)},
    xtick={256,1024,4096,16384,65536},
    xticklabels={$2^{8}$,$2^{10}$,$2^{12}$,$2^{14}$,$2^{16}$,$2^{18}$},
    ytick={100,200},
    yticklabels={0.1,0.2},
    ymin=0, ymax=250,
    grid=both,
    grid style={gray!20},
    legend style={
        at={(0.0,1.0)},
        anchor=north west,
        legend columns=1,
        font=\scriptsize,
        draw=none,
        fill=none,
        legend cell align=left,
    },
    tick label style={font=\scriptsize},
    label style={font=\scriptsize},
    thick,
    mark options={scale=0.9},
    scale only axis,
    ylabel style={yshift=-5pt},
]

\addplot[
    color=colPOM,
    mark=square*,
    line width=1pt
] coordinates {
    (256, 5.415)
    (512, 5.403)
    (1024,5.411)
    (2048,5.437)
    (4096,5.594)
    (8192,5.985)
    (16384,6.934)
    (32768,8.970)
    (65536,13.171)
};
\addlegendentry{\pomicon\ PoM}

\addplot[
    color=colTrans,
    mark=o,
    line width=1pt
] coordinates {
    (256,  23.002)
    (512, 35.250)
    (1024,59.793)
};
\addlegendentry{Self-attention}
 \addplot[
     color=colTrans,
     dotted,
     line width=1pt,
     forget plot
 ] coordinates {
     (1024,59.793)
     (1500,100)
 };
\node[text=red, anchor=west] at (axis cs:1500,100) {\scriptsize OOM};

\addplot[
    color=colFlash,
    mark=o,
    line width=1pt
] coordinates {
    (256, 4.756)
    (512, 5.681 )
    (1024,7.511)
    (2048,11.099)
    (4096,18.207)
    (8192,32.317)
    (16384,60.612 )
    (32768,117.260)
    (65536,230.486)
};
\addlegendentry{Self-attention w/ FlashAttention}

\addplot[
    color=colPOM,
    mark=square*,
    line width=1pt,
    forget plot
] coordinates {
    (256, 5.415)
    (512, 5.403)
    (1024,5.411)
    (2048,5.437)
    (4096,5.594)
    (8192,5.985)
    (16384,6.934)
    (32768,8.970)
    (65536,13.171)
    };
\end{semilogxaxis}
\end{tikzpicture}

{\centering \footnotesize a) Speed of attention module (batch of $2^{18}$ tokens)}
\end{minipage}
\hfill
\begin{minipage}[t]{0.33\textwidth}
\centering
\hspace*{2mm}
\begin{tikzpicture}
\begin{semilogxaxis}[
    width=.9\linewidth,
    height=4.0cm,
    xlabel={\scriptsize Sequence length (tokens)},
    ylabel={\scriptsize Execution time (s/batch)},
    xtick={1024,4096,16384,65536},
    xticklabels={$2^{10}$,$2^{12}$,$2^{14}$,$2^{16}$},
    ymin=0, ymax=2600,
    grid=both,
    grid style={gray!20},
    legend style={
        at={(0.02,0.98)},
        anchor=north west,
        legend columns=1,
        font=\scriptsize,
        draw=none,
        fill=none,
        legend cell align=left
    },
    ytick={1000,2000},
    yticklabels={1,2},
    tick label style={font=\scriptsize},
    label style={font=\scriptsize},
    thick,
    mark options={scale=0.9},
    scale only axis,
    ylabel style={yshift=-3pt},
]

\addplot[
    color=colFlash,
    mark=o,
    line width=1pt
] coordinates {
    (1024,115.7618)
    (2048,152.1319)
    (4096,225.6129)
    (8192,373.0753)
    (16384,669.9945)
    (32768,1272.3304)
    (65536,2506.8214)
};
\addlegendentry{GPT2-S w/ FlashAttention}

\addplot[
    color=colPOM,
    mark=square*,
    line width=1pt,
    mark options={
        fill=colFlash,
        draw=colFlash 
    }
] coordinates {
    (1024,84.6557)
    (2048,100.2772)
    (4096,131.0474)
    (8192,193.6755)
    (16384,318.2232)
    (32768,570.6613)
    (65536,1082.3575)
};
\addlegendentry{GPPoM2-S Hyb}

\addplot[
    color=colPOM,
    mark=square*,
    line width=1pt
]
coordinates {
    (1024,69.0017)
    (2048,73.9689)
    (4096,84.2516)
    (8192,105.6855)
    (16384,147.1300)
    (32768,229.4805)
    (65536,393.1584)
};
\addlegendentry{\pomicon\ GPPoM2-S}

\end{semilogxaxis}
\end{tikzpicture}
{\hspace*{4mm}\centering  \footnotesize b) Speed of full model (batch of $2^{16}$ tokens)}
\end{minipage}
\hfill
\begin{minipage}[t]{0.33\textwidth}
\centering
\begin{tikzpicture}
\newcommand{\xtask}[3]{%
  \begin{tabular}{@{}c@{}}
     ~\\[-5mm]
    \rule{0pt}{1.1em}#1\\[-1.4mm]
    \rule{0pt}{1.0em}{\tiny #2}\\[-1.4mm]
    \rule{0pt}{1.0em}{\tiny #3}
  \end{tabular}%
}
\begin{axis}[
    width=.78\linewidth,
    height=3.55cm,
    ybar,
    ymin=0, ymax=115,
    ytick=\empty,
    symbolic x coords={NLP,OCR,3D,EO,ImGen},
    xtick={NLP,OCR,3D,EO,ImGen},
    xticklabels={\xtask{NLP}{\tiny{HellaSwag}}{\tiny{AccNorm $\uparrow$}},
    \xtask{OCR}{LAM}{WER $\downarrow$}, \xtask{3D}{SemKITTI}{mIoU$\uparrow$},    \xtask{EO}{PASTIS}{mIoU$\uparrow$},
    \xtask{ImgGen}{ImgNet}{FID $\downarrow$}},
    bar width=5pt,
    enlarge x limits=0.15,
    grid=both,
    thick,
    grid style={gray!20},
    legend style={
        at={(0.5,-0.22)},
        anchor=north,
        legend columns=3,
        font=\scriptsize,
        draw=none,
        fill=none
    },
    tick label style={font=\scriptsize},
    scale only axis,
]

% --- Categories with 3 bars: NLP, OCR, 3D ---
\addplot[
    ybar,
    bar shift=-7pt,
    fill=colTrans,
    draw=colTrans,
    nodes near coords,
    nodes near coords style={font=\tiny, xshift=-2pt, yshift=-1pt},
    point meta=explicit symbolic
] coordinates {
    (NLP,100) [33.3]
};

\addplot[
    ybar,
    bar shift=-7pt,
    fill=colTrans,
    draw=colTrans,
    nodes near coords,
    nodes near coords style={font=\tiny, xshift=-1pt, yshift=-0pt},
    point meta=explicit symbolic,
    forget plot
] coordinates {
    (OCR,100) [2.8]
};

\addplot[
    ybar,
    bar shift=-7pt,
    forget plot,
    fill=colTrans,
    draw=colTrans,
    nodes near coords,
    nodes near coords style={font=\tiny, xshift=-3pt,
    },
    point meta=explicit symbolic,
    forget plot
] coordinates {
    (3D,100) [67.2]
};

\addplot[
    ybar,
    bar shift=0pt,
    fill=colPOM,
    draw=colPOM,
    postaction={
        pattern={Lines[angle=45,distance=3pt,line width=1.2pt]},
        pattern color=colTrans,
    },
    nodes near coords,
    nodes near coords style={font=\tiny, yshift=+0pt, yshift=1pt},
    point meta=explicit symbolic,
    forget plot
] coordinates {
    (NLP,101.5) [33.8]
};

\addplot[
    ybar,
    bar shift=0pt,
    fill=colPOM,
    draw=colPOM,
    postaction={
        pattern={Lines[angle=45,distance=3pt,line width=1.2pt]},
        pattern color=colTrans,
    },
    nodes near coords,
    nodes near coords style={font=\tiny, xshift=+1pt, yshift=0pt},
    point meta=explicit symbolic,
    forget plot
] coordinates {
    (OCR,100) [2.8]
};

\addplot[
    ybar,
    bar shift=0pt,
    fill=colPOM,
    draw=colPOM,
    postaction={
        pattern={Lines[angle=45,distance=3pt,line width=1.2pt]},
        pattern color=colTrans,
    },
    nodes near coords,
    nodes near coords style={font=\tiny, yshift=+0pt, yshift=2pt},
    point meta=explicit symbolic
] coordinates {
    (3D,99.9) [67.6]
};

\addplot[
    ybar,
    bar shift=+7pt,
    fill=colPOM,
    draw=colPOM,
    nodes near coords,
    nodes near coords style={font=\tiny, yshift=0pt, xshift=+0pt},
    point meta=explicit symbolic
] coordinates {
    (NLP,82.3) [27.4]
    (OCR,89.2) [3.2]
};

\addplot[
    ybar,
    bar shift=+7pt,
    fill=colPOM,
    draw=colPOM,
    nodes near coords,
    nodes near coords style={font=\tiny, yshift=-2pt, xshift=+2pt},
    point meta=explicit symbolic
] coordinates {
    (3D,99.1) [66.6]
};

% --- Categories with only 2 bars: EO, ImGen ---
\addplot[
    ybar,
    bar shift=-3.5pt,
    fill=colTrans,
    draw=colTrans,
    nodes near coords,
    nodes near coords style={font=\tiny, xshift=+3.5pt},
    point meta=explicit symbolic
] coordinates {
    (EO,100) [64.6]
    (ImGen,100) [17.2]
};

\addplot[
    ybar,
    bar shift=3.5pt,
    fill=colPOM,
    draw=colPOM,
    nodes near coords,
    nodes near coords style={font=\tiny, xshift=1.5pt},
    point meta=explicit symbolic
] coordinates {
    (EO,100.4)
    (ImGen,100)
};

\legend{Transformer,Hybrid,\pomicon\ PoM}
\end{axis}
\end{tikzpicture}

{\centering \footnotesize c) Performance across domains}
\end{minipage}

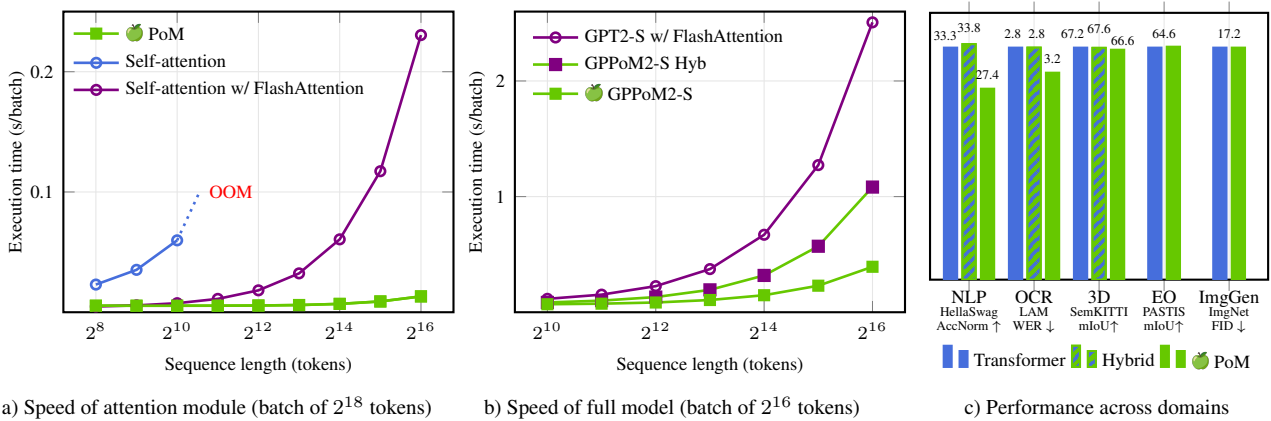
\captionof{figure}{\textbf{Polynomial Mixer (\pomicon PoM).}
Inference time on an H100 as a function of sequence length, shown for the PoM module alone (a) and within a full Transformer pipeline (b). 
PoM scales nearly linearly and is significantly faster than self-attention for long sequences, even with FlashAttention.
As measured for NLP, OCR, 3D point cloud segmentation, Earth Observation analysis,  and image generation (c), replacing attention with PoM or hybrid variants preserves accuracy while improving scalability.}

    \label{fig:teaser}
    ~\\[1mm]
 }
 ]

\begin{abstract}
This paper introduces the \emph{Polynomial Mixer} (PoM), a novel token mixing mechanism with linear complexity that serves as a drop-in replacement for self-attention. PoM aggregates input tokens into a compact representation through a learned polynomial function, from which each token retrieves contextual information. We prove that PoM satisfies the \emph{contextual mapping property}, ensuring that transformers equipped with PoM remain universal sequence-to-sequence approximators. We replace standard self-attention with PoM across five diverse domains: text generation, handwritten text recognition, image generation, 3D modeling, and Earth observation. PoM matches the performance of attention-based models while drastically reducing computational cost when working with long sequences. The code is available at \url{https://github.com/davidpicard/pom}.
\end{abstract}

\section{Introduction}
\epigraph{One apple is a supper; one apple is life.}{- \emph{Les Misérables, Book Fourth}, {Victor Hugo}}

Multi-Head Attention (MHA)~\cite{vaswani17nips} is the key ingredient to recent advances in AI, powering large language models~\cite{radford2019language,dubey2024llama,team2024gemma}, generative models for images and videos~\cite{peebles23iccv,ma24eccv,yang2024cogvideox}, as well as breakthroughs in computer vision~\cite{dosovitskiy2020image} and speech processing~\cite{defossez2024moshi,zhang2023google,baevski2020wav2vec,chiu2022bestrq}.
However, MHA has a quadratic complexity with respect to the input sequence length, severely limiting the attainable context size. Numerous sub-quadratic approximations have been proposed, yet most trade efficiency for either performance or generality. For instance, Linformer~\cite{wang2020linformer} requires a fixed sequence length, while State Space Models (SSMs)~\cite{pei2024efficientvmambaatrousselectivescan} impose a causal ordering that does not apply to all modalities.

In this paper, we question whether quadratic complexity is truly necessary to match the performance of MHA on sequence-to-sequence tasks. 
Our approach draws inspiration from the long-standing use of high-order moments and polynomial features in representation learning~\cite{picard2011improving,jacob2019metric}, which we revisit here as an alternative to attention. 
We introduce the Polynomial Mixer (PoM), a novel module with linear complexity that achieves comparable performance to MHA across diverse, real-world tasks. PoM aggregates the entire input sequence into a compact representation using a learned polynomial function, computed in linear time. Each token then retrieves relevant contextual information from this summary through a cross-attention-like mechanism.

The \emph{contextual mapping property}~\cite{Yun20ICLR} is a key aspect of MHA enabling modern transformers to be universal sequence-to-sequence approximators. This property is rarely satisfied by standard neural network components; for instance, MLPs, CNNs, and SSMs do not exhibit it. We show that our proposed Polynomial Mixer (PoM) has the contextual mapping property, allowing it to replace attention blocks in transformers without sacrificing expressivity.

We validate the efficiency and expressivity of PoM across five diverse tasks from different domains: text generation, image generation, optical character recognition, 3D point cloud segmentation, and multi-temporal satellite images analysis. As shown in \cref{fig:teaser}, PoM matches the performance of MHA while providing a significant speedup on long sequences. Our main contributions are: 
\begin{itemize}
    \item We propose PoM, a token mixing mechanism with linear complexity in the number of tokens.
    \item We prove that PoM satisfies the contextual mapping property, ensuring transformers equipped with PoM remain universal sequence-to-sequence approximators.
    \item We demonstrate that PoM achieves performance on par with MHA across five distinct domains, while being substantially more efficient.
\end{itemize}

\section{Related Work}

We review prior work on efficient attention mechanisms and alternative architectures that aim to replace or bypass self-attention.

\paragraph{Faster Attention.}
Since the introduction of Transformers~\cite{vaswani17nips}, many efforts have been made to reduce the quadratic complexity of MHA~\cite{child2019generating,kitaev2020reformer,wang2020linformer}.
Notably, methods like Reformer~\cite{kitaev2020reformer} use fast approximate neighbors to reduce the size of the attention matrix based on the assumption that most tokens will have zero attention. 
To go further, Linformer~\cite{wang2020linformer} proposes to compute an explicit low rank projection of the keys and the values to reduce the complexity of MHA from the size of the sequence $n$ to an arbitrary chosen number $k \ll n$.
The main drawback of such approach is that $n$ and $k$ are fixed, which means that the model can no longer process sequences of varying length.
With the advent of Large Language Models and their ability to process extremely long sequences~\cite{achiam2023gpt,team2023gemini,dubey2024llama}, recent efforts have been put on more efficient implementations such as Flash-Attention~\cite{dao22nips,dao2023flashattention} or KV-cache~\cite{brandon2024reducing,luohe2024keep} which seem sufficient for text.
However for visual content, the sequence length grows quadratically with the resolution. Because MHA is also quadratic in the number of tokens, this leads to quartic computational and memory complexity.

\paragraph{Alternatives to Attention.}
Alternatively, some attempts have been made to remove the Multi-Head Attention entirely, such as in MLP-Mixer~\cite{tolstikhin21nips} and ResMLP~\cite{touvron2022resmlp} that replace MHA with simple projection on the transpose tensor (\textit{i.e.}, considering the sequence dimension as features).
These approaches have been shown to obtain competitive results, but similarly to Linformer, they imply a fix sequence length since this length is now an intrinsic dimension of the projection in the transpose direction.
More recently, State-Space Models (SSM)~\cite{gu21iclr,gu21nips} have become the focus of recent work especially in language modeling~\cite{zuo2024falcon,dao24icml,lieber2024jamba,glorioso2024zamba}.
SSM are recurrent models, which is highly beneficial for language modeling because of the causal property of text: the complexity to generate the next token becomes constant.

In visual content however, there is no such natural causality pattern in the spatial dimensions.
Attempt to use such models for vision tasks have been successful~\cite{zhu2024vision,liu2024vmambavisualstatespace,pei2024efficientvmambaatrousselectivescan}, albeit at the cost of enforcing an arbitrary 1-dimensional scan order of the tokens that does not encode well the 2D nature of an image.
In image generation with diffusion~\cite{hu24eccv,yan24cvpr}, since the model has to be iterated, this results in a doubly sequential processing (space and iterations) that does not benefit from the parallel nature of processing images.
For video however, the causal aspect is natural over the time dimension, recurrent approaches may be more efficient.
Overall, we question whether a linear time algorithm can achieve the performances of the quadratic attention without imposing causal constraints as SSMs or a fixed sequenced size as Linformer or ResMLP.

\section{Method}
In this section, we first describe PoM, our proposed sequence-to-sequence drop-in replacement for MHA (\cref{sec:pom}). We then propose an extension to causal sequences (\cref{sec:causal}) and a theoretical analysis (\cref{sec:analysis})
%=============================
\subsection{Polynomial Mixer and Polymorpher}
%=============================
\label{sec:pom}

\begin{figure*}[t]
    \centering
    \begin{tabular}{c@{\;}c}
    \hspace{-8.8pt} % Pull the figure to the left, as the white spacing was too large
    \begin{minipage}{0.332\linewidth}
    \captionof{figure}{\textbf{Polynomial Mixer.} 
An input sequence $X$ of $n$ tokens with dimension $d$ follows two parallel paths. 
In the top path, each token is first expanded from dimension $d$ to $D$, then transformed through a polynomial of degree $k$ with coefficients $\bm{\alpha}$ (all multiplications and exponentiations are element-wise). 
Dimension $k$ is collapsed by summation, yielding a shared state $H(X) \in \mathbb{R}^D$. 
In the bottom path, each token produces gating coefficients in $[0,1]^D$, which query information from $H(X)$ via element-wise multiplication (broadcasted across $n$). 
The result is finally projected back to the original dimension $d$.}
    \label{fig:pom}
     \end{minipage}
     &
    \begin{minipage}{0.65\linewidth}
    \resizebox{\linewidth}{!}{
    \def\xinput{0}%
\def\xh{3.5}%
\def\xpoly{6}%
\def\xquerry{6}%
\def\xpolyout{8}%
\def\xpom{10}%
\def\xdot{11}%
\def\xmix{12.25}%
\def\xout{15}%
\def\ymid{0}%
\def\yup{2}%
\def\ydown{-2}%
\def\ysep{1.2}%
\def\hdot{0.3}%
\providecommand{\tokenbars}{10}%
\definecolor{colorA}{RGB}{66,198,184}%    % light teal
\definecolor{colorB}{RGB}{254,170,94}%    % light orange
\definecolor{colorC}{RGB}{110,150,230}% % light blue
\definecolor{colorD}{RGB}{214,106,205}% % light magenta
\definecolor{colorE}{RGB}{205,205,205}%   % light gray% light gray
\begin{tikzpicture}
%\draw[use as bounding box, red] (-0.5,-3) rectangle (16,3);
%input tokens
\node [draw=none] (input) at (\xinput, \ymid) {\tokengrid{1}{4}{1}{colorA}};

%after h(Wh X)
\node [draw=none] (prepoly) at (\xh, \yup) {\tokengrid{1}{4}{2}{colorB}};

%poly
\node [very thick, draw=black, rounded corners] (poly) at (\xpoly, \yup) {%
$\begingroup
\begin{array}{@{}r@{\;}c@{\;}l@{\;}l@{}}
  \alpha_1 & \odot & x   & + \\
  \alpha_2 & \odot & x^2 & + \\
  \multicolumn{4}{c}{\cdots} \\[-1mm]
  \alpha_k & \odot & x^k &
\end{array}
\endgroup$%
};

%after polynom
\node [draw=none] (postpoly) at (\xpolyout, \yup) {\tokengrid{1}{4}{2}{colorB}};

%pom
\node [draw=none] (pom) at (\xpom, \yup) {\tokengrid{1}{1}{2}{colorB}};

%querries
\node [draw=none] (querries) at (\xquerry, \ydown) {\tokengrid{1}{4}{2}{colorC}};

%dot
\node [very thick, circle, minimum height=\hdot cm, draw=black] (dot) at (\xdot, \ymid) {
\begin{tikzpicture}
    \fill[black] (0,0) circle (0.5 mm);
\end{tikzpicture}
};

%after mix
\node [draw=none] (mix) at (\xmix, \ymid) {\tokengrid{1}{4}{2}{colorD}};

%out
\node [draw=none] (out) at (\xout, \ymid) {\tokengrid{1}{4}{1}{colorE}};

%edges
\draw [very thick, rounded corners, ->] (input) -- ++ (1,0) |- node [pos=1.0, above, draw=none, anchor=south east] {$h(W_h \times x)$}  (prepoly);

\draw [very thick, rounded corners, ->] (prepoly) -- (poly) --  (postpoly) -- node [pos=0.5, above, draw=none] {$\updownarrow\! \sum$} (pom) -| (dot);

\draw [very thick, rounded corners, ->] (input) -- ++ (1,0) |- node [pos=0.75, above, draw=none] {$\sigma(W_s \times x)$} (querries);

\draw [very thick, rounded corners, ->] (querries) -| (dot);

\draw [very thick, rounded corners, ->] (dot) -- (mix) -- node [pos=0.5, above, draw=none] {$h(W_o \times x)$} (out);

%shapes

\node [draw=none, below=-1.5mm of input] {$d \times n$};
\node [draw=none, above=-1.5mm of input] {$X$};

\node [draw=none, below=-1.5mm of prepoly] {$D \times n$};

\node [draw=none, below=-1.5mm of postpoly] {$D \times n$};

\node [draw=none, below=-1.5mm of pom] {$D$};
\node [draw=none, above=-1.5mm of pom] {$H(X)$};

\node [draw=none, above=-1.5mm of querries] {$D \times n$};

\node [draw=none, below=-1.5mm of mix] {$D \times n$};

\node [draw=none, below=-1.5mm of out] {$d \times n$};
\node [draw=none, above=-1.5mm of out] {$\PoM(X)$};

%\node [draw=none, below=-1.0mm of poly] {\shortstack{applied\\token-wise}};
\end{tikzpicture}
    }
    \end{minipage}

    \end{tabular}
\end{figure*}
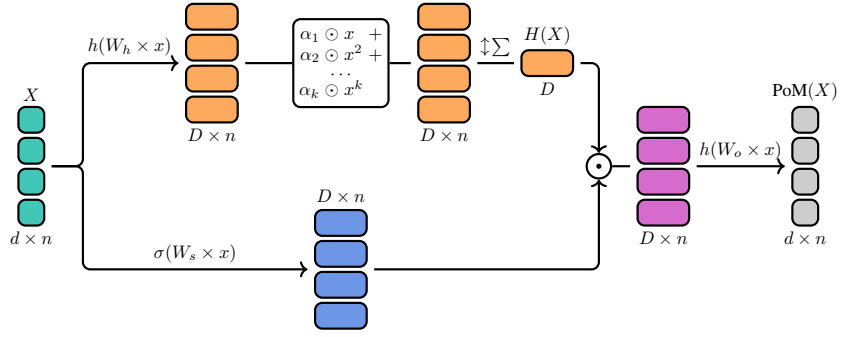

\paragraph{Polynomial Mixer.}
We consider an input sequence of $n$ tokens of dimension $d$, stored in a tensor $X \in \mathbb{R}^{d\times n}$.
Unlike MHA, which computes all pairwise interactions between tokens, the Polynomial Mixer relies on a shared \emph{state representation} $H(X)$ that serves as a common memory accessible to all tokens. This state aggregates information from the entire sequence after mapping tokens into a higher-dimensional space $\mathbb{R}^D$ through a learned polynomial function—hence the name \emph{Polynomial Mixer}.

Formally, the sequence $X \in \mathbb{R}^{d\times n}$ is summarized into a single representation $H(X) \in \mathbb{R}^{D\times 1}$, defined as a polynomial of degree $k$:
\begin{align}
    H(X) &= \left[\sum_{p=1}^k \bm{\alpha}_p \odot h(W_h X)^p\right]\bm{1}~, 
\end{align}
where $D$ is the internal feature dimension, $\bm{\alpha}\in\mathbb{R}^{D\times k}$ are the polynomial coefficients, $h$ is an activation function, $\odot$ denotes the element-wise (Hadamard) product with broadcasting, and $\bm{1}$ is a vector of ones used to sum over the sequence dimension $n$. The exponentiation is also applied element-wise. The matrix $W_h \in \mathbb{R}^{D\times d}$ is a learnable parameter that mixes dimensions to ensure $H(X)$ is composed of inter-dimension polynomials instead of just monomials.

Each token then queries the shared representation $H(X)$ using
the gating coefficient $S(X) = \sigma(W_s X) \in [0,1]^{D\times n}$ and the resulting information is projected back into the original space by $W_o$:
\begin{align} \label{eq:pom}
    \PoM(X) &= W_o \left[\sigma(W_s X) \odot H(X)\right]~,
\end{align}
where $\sigma$ is the sigmoid function, and $W_o \in \mathbb{R}^{d\times D}$ and $W_s \in \mathbb{R}^{D\times d}$ are learnable parameters.  
All operations are linear in the sequence length $n$.

\if 1 0
We define a \emph{Polymorpher} block as a sequence-to-sequence function mapping $\mathbb{R}^{d\times n}$ to $\mathbb{R}^{d\times n}$, composed of two residual blocks, a \emph{Polynomial Mixer} and a feed-forward block. 

For a sequence $X\in \mathbb{R}^{d\times n}$, the Polynomial Mixer ($\PoM$) shown on Figure \ref{fig:pom} is defined as follows:
\begin{align}
    \PoM(X) &= W_o \left[\sigma(W_s X) \odot H(X)\right],\text{ with}\\
    H(X) &= \left[\sum_{p=1}^k \bm{\alpha}_p \odot h(W_h X)^p\right]\bm{1}, 
\end{align}
where $k$ is the degree and $\bm{\alpha}\in\mathbb{R}^{D\times k}$ are the coefficients of the polynomial, $\sigma$ is the sigmoid function, $h$ an activation function, $\odot$ the element-wise (Hadamard) product with possible broadcasting, and $\bm{1}$ a vector of the appropriate dimension filled with ones. The matrices $W_o\in \mathbb{R}^{d\times D}$, $W_s\in \mathbb{R}^{D\times d}$ and $W_h \in \mathbb{R}^{D\times d}$ are the learnable parameters of the Polynomial Mixer.
\fi

\paragraph{PolyMorpher.}
Drawing inspiration from Transformer architectures, we define a \emph{PolyMorpher} block $P$ that alternates residual Polynomial Mixer and feed-forward layers:
\begin{align}
    \text{P}(X) = X + \PoM(X) + \text{FF}(X + \PoM(X)),
\end{align}
where $\text{FF}(X)$ denotes a standard two-layer feed-forward network.  
The PolyMorpher thus acts as a \emph{drop-in replacement} for Transformer blocks based on MHA, fulfilling the same sequence-to-sequence mapping role. The key difference lies in its parametrization: while Transformers are configured by the number of heads and their dimension, the Polymorpher is defined by the  degree $k$ and the internal dimension $D$ of its polynomial.

%============================
\subsection{Polymorpher for Causal Sequences}
\label{sec:causal}
%============================
We consider a sequence with a causal structure encoded by a binary mask $M \in \{0,1\}^{n \times n}$, where $M_{s,t} = 1$ if and only if token $s$ is allowed to use information from token $t$ to update its representation.  
The Polynomial Mixer can be readily adapted to handle such causal dependencies by defining a distinct state representation $H(X)$ for each time step of the sequence:
\begin{align}
    H(X) = \left[\sum_{p=1}^k \bm{\alpha}_p \odot h(W_h X)^p\right]M^\top~. 
\end{align}
Here, $H(X) \in \mathbb{R}^{D \times n}$, and the formulation in \Cref{eq:pom} remains unchanged.

\paragraph{Causal Sequence.}
In the special case of causal sequences, the mask $M$ is a lower triangular matrix. 
This allows computing the state representation at time step $t$, denoted $H(X)_t$, through the following iterative process:
\begin{align}
    H(X)_{t} &= \sum_{s\leq t} \left[\sum_{p=1}^k \bm{\alpha}_p \odot h(W_h X)^p\right]_{s},\\
    &= H(X)_{t-1} + \left[\sum_{p=1}^k \bm{\alpha}_p \odot h(W_h X)^p\right]_{t}~.
\end{align}
This formulation enables $\mathcal{O}(1)$ inference complexity in the autoregressive setting, a property shared with recurrent networks but not with transformers. 
Similar to SSMs, Polymorphers combine the advantages of both worlds: they can be trained on the entire sequence in parallel while supporting efficient recursive inference.

\paragraph{Block Causal Sequence.}
Polymorphers can also model \emph{block-causal} sequences, as found in spatio-temporal data such as videos or satellite image time series. 
In this setting, each token can attend to all tokens within the same frame, as well as to tokens from preceding frames.
Let $M$ denote a block-causal mask for a block size $K$:
\begin{align}
    M_{i,j} = 1 \text{ if } j \leq \lceil i / K \rceil K \text{ else } 0.
\end{align}
The state representation $H$ can then be computed as
\begin{align}
\nonumber H(X)_{t} = &H(X)_{\lfloor t/K\rfloor K} \\
    &+ \sum_{s = \lfloor t/K\rfloor K}^{\lceil t / K \rceil K}\left[\sum_{p=1}^k \bm{\alpha}_p \odot h(W_h X)^p\right]_{s}.
\end{align}
This formulation allows processing tokens in blocks during inference, reducing memory requirements while preserving causal dependencies.

%=======================%
\subsection{Theoretical Analysis}
\label{sec:analysis}
%=======================%

We first show that $\PoM$ is \emph{equivariant}, which means that permutations in the input sequence result in permuted outputs. This is a key property that made transformers popular and does not hold for other architectures like convolutions or recurrent networks.

\begin{proposition}[Permutation equivariance]
    A Polynomial Mixer is permutation equivariant, \ie, let $X \in \mathbb{R}^{d\times n}$ be a set of vectors and $P$ a column permutation matrix, then $\PoM(XP) = \PoM(X)P$.
\end{proposition}
\begin{proof}
    For a permutation $P$, we have 
    \begin{align}\PoM(XP) = W_o \left[ \sigma(W_s X P) \circ H(X P)\bm{1}^\top \right]. 
    \end{align}
    We have that $H(X P) = H(X)$ because the sum is permutation invariant, and $\sigma(W_s X P) = \sigma(W_s X) P$ because $\sigma$ is an element-wise operation. 
    Noticing that $H(X)\bm{1}^\top$ has all identical columns allows us to move $P$ outside of the brackets to conclude the proof.
\end{proof}

More importantly, we can also prove a universal approximation theorem for Polymorphers similar to what is well known for Transformers~\cite{Yun20ICLR}. As the polynomial mixer is equivariant, it requires the use of positional encoding (learned for simplicity reasons), which also underlines the similarity between $\PoM$ and MHA. 

We use the following standard definition of distance between functions that map sequences to sequences. Given two functions $f$ and $g: \mathbb{R}^{d_n}\rightarrow \mathbb{R}^{d_n}$ and an integer $A\leq p\leq \infty$, we define the distance $d_p$ as:
\begin{align}
    d_p(f, g) = \left(\int \| f(X) - g(X)\|_p^pdX\right)^{1/p}.
\end{align}

The following theorem holds:
\begin{theorem}[Universal approximation]
    Let $1 \leq p \leq \infty$ and $\epsilon > 0$, then for any given $f\in \mathcal{F}$ the set of continuous functions that map a compact domain in $\mathbb{R}^{d\times n}$ to $\mathbb{R}^{d\times n}$, there exists a Polymorpher $g$ with learned positional encoding such that $d_p(f, g) \leq \epsilon$.
\end{theorem}

The proof follows exactly the same scheme as in ~\cite{Yun20ICLR}, where most of the heavy lifting is done by the feed-forward networks. Their main argument is to show that MHA can map every token in the sequence to a unique value that depends on the entire sequence, and then the feed-forward blocks can map those unique values to the desired output. In our case, we just have to ensure that the Polynomial Mixer has the same properties as MHA, which is obtained using the following lemma:

\begin{lemma}[Contextual mapping (informal)]
    There exists $k > 0$ for which any Polynomial Mixer $q$ of degree $k$ is a contextual mappings on $\mathbb{R}^{d\times n}$, that is:
    \begin{itemize}
        \item For any $X \in \mathbb{R}^{d\times n}$ with different entries, $q(X)$ has different entries.
        \item For any $X, X' \in \mathbb{R}^{d\times n}$ that differ at least by one element, then all entries of $q(L)$ and $q(L')$ are different.
    \end{itemize}
\end{lemma}

The proof is deferred to the appendix and primarily uses the fact that a sufficiently high degree polynomial is uniquely defined by a sequence of point-wise evaluation. As noted in \cite{Yun20ICLR}, having the contextual mapping property is not so common as it requires to summarize uniquely the context while preserving the identity of the current token.

With these results, we show that a Polymorpher is as potent as a Transformer for sequence modeling.
\begin{figure*}[t]
    \centering
    \input{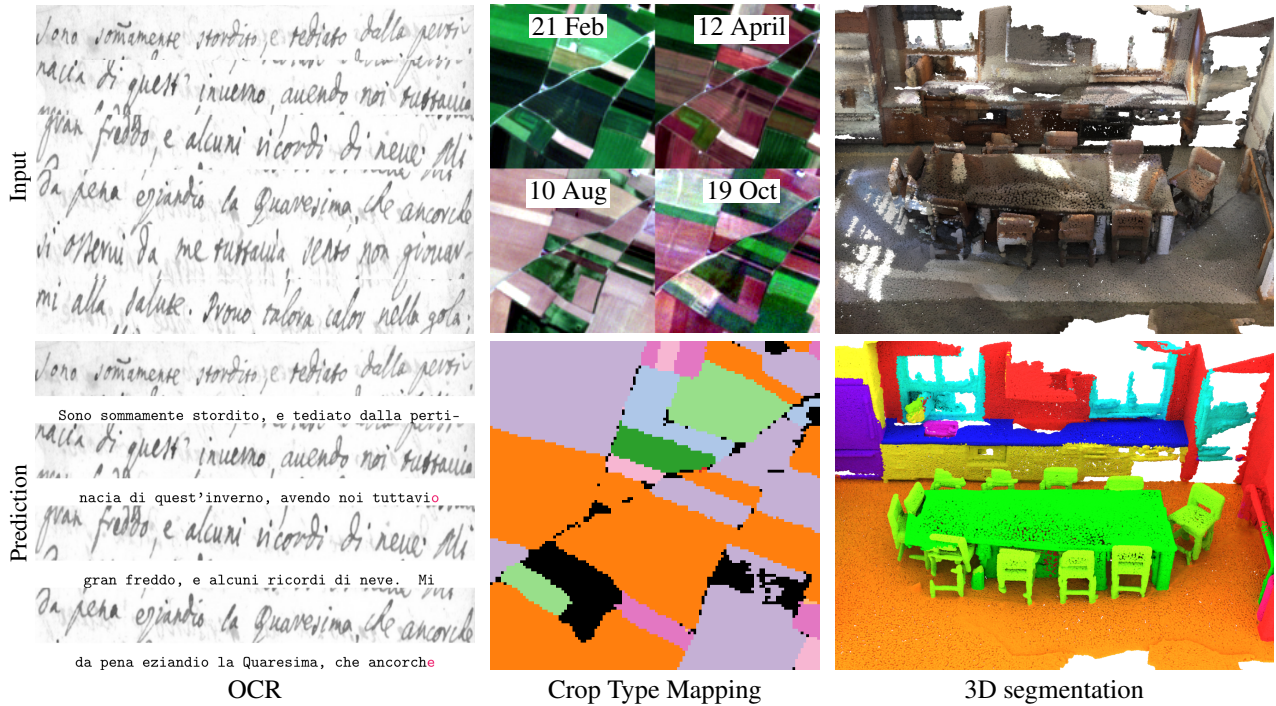}
    % \vspace{-2mm}
    \caption{{\bf Evaluation Across Domains.} We evaluate PoM on various tasks from multiple domains, simply replacing some or all the self-attention blocks in SOTA Transformer-based models. \emph{Left}: For OCR, given a page of handwritten text (top) split into lines, the goal is to recognize each character and group them into words (bottom). \emph{Middle}: For Earth observation, given a time series of satellite images (top), the goal is to classify each pixel as a crop type by the end of the series (bottom). \emph{Right}: For 3D segmentation, given a 3D point cloud (top), the goal is to classify each point according to classes (bottom).}
    \label{fig:all_tasks}
\end{figure*}

%===============================
\subsection{Complexity and Efficiency Analysis}
\label{sec:complex}
%===============================
\paragraph{When is PoM faster than MHA?}
For a sequence of length $n$, the cost of PoM breaks down as follows: $2dDn$ multiplications for  input projections, $nkD$ multiplications to compute the polynomial, and and $nD$ sums to aggregate it. The selection operation costs $nD$ multiplications, and the output projection costs  $nDd$. Assuming that the cost of additions is negligible in front of that of multiplications, this gives us a total cost of $3ndD + (k+1)nD$ multiplications.

In comparison, the cost of input projections for attention is $3nd^2$ multiplications. Then, the cost of computing the attention matrix is $dn^2$ multiplications and aggregating the values is also $dn^2$ multiplications. Finally, the output projection cost is $nd^2$. This sums to a total cost of $4nd^2 + 2dn^2$.

Solving for $n$ leads to PoM being faster when
\begin{align}
    n \geq \frac{D(3d+k+1) - 4d^2}{2d}~,
\end{align}
with the polynomial degree $k$ and the hidden dimension $D$ setting different trade-offs.
In our experiments, we found that $D=2d$ and $k=2$ is sufficient to match the performances of attention. In this case, PoM has a complexity advantage for $n\geq d+3$ which usually amounts to sequences of sizes greater than 1k tokens ($d=512$--$1024$).

Below this range, highly optimized attention kernels can remain competitive; however, as $n$ grows, the cost of MHA increases quadratically while PoM grows linearly, making PoM significantly more efficient for long-context applications.  
This trend is especially pronounced in settings such as high-resolution image or video modeling, geospatial sequences, or autoregressive generation over long documents, where PoM’s linear scaling yields substantial speedups.

\paragraph{Comparing with Flash-Attention.}
Highly optimized attention implementations such as FlashAttention~\cite{dao2022flashattention} achieve near-linear runtime for moderate sequence lengths through clever memory management and custom CUDA kernels. 
In contrast, our PoM implementation is written entirely in high-level PyTorch, requiring no compilation or specialized code. 
Despite this, PoM already demonstrates superior scalability for long sequences and large images, but remain slower for smaller ones. 
We argue that directly comparing new methods with highly-optimized implementations should be interpreted cautiously.  
This paper presents a proof of concept, showcasing PoM’s efficiency and versatility across several domains and multiple tasks. We expect further gains with dedicated low-level optimization which are out of the scope of our proof of concept.

\section{Experiments}
We evaluate PoM across a diverse set of tasks and modalities: natural language processing (\cref{sec:nlp}), optical character recognition (\cref{sec:ocr}), Earth observation analysis (\cref{sec:eo}), 3D point cloud segmentation (\cref{sec:3d}), and image generation (\cref{sec:imgen}). Examples for the visual tasks are shown in Figure~\ref{fig:all_tasks}.
This diversity highlights that PoM can serve as a drop-in, compute-efficient replacement for the standard self-attention transformer block. 

We also explore the validity of hybrid models which retain some attention layers. In all our experiments, we do not introduce major architectural changes compared to the original models, but we parametrize PoM and the MLP in the blocks so as to keep the total number of parameters close to the original one. We also do not deviate from their original training recipe, except for experiments where changing the number of tokens would create training instabilities both for attention and PoM.

\begin{table}[t]
    \centering
    \begin{tabular}{c}
    \begin{subtable}{\linewidth}
    \resizebox{\linewidth}{!}{\centering\footnotesize
    \begin{tabular}{l@{\,}!{\color{gray!20}\vrule}c@{\,}c@{\,}c@{\,}c@{\,}c@{\,}}
    \toprule
    \multirow{2}{*}{\textbf{ Model }} & \multirow{2}{*}{\makecell[c]{\textbf{Val} \\ \textbf{Loss}↓}} & \multirow{2}{*}{\makecell[c]{\textbf{ARC-E} \\ \textbf{Acc.Norm}↑}} & \multirow{2}{*}{\makecell[c]{\textbf{HellaSwag} \\ \textbf{Acc.Norm}↑}} & \multirow{2}{*}{\makecell[c]{\textbf{Winogrande} \\ \textbf{Acc.}↑}} & \multirow{2}{*}{\makecell[c]{\textbf{MMLU} \\ \textbf{Acc.}↑}}\\
    &&&&&\\
    \midrule
    {\color{gray!90}GPT2-S (OpenAI)}    & {\color{gray!90}-} & {\color{gray!90}42.8} & {\color{gray!90}31.6}  & {\color{gray!90}50.0}  & {\color{gray!90}26.1}  \\
    GPT2-S 12MHA     & \textbf{3.29} & \textbf{29.4}  & {33.3}  & 49.4  & 24.6  \\
    \pomicon GPPoM2-S & 3.88 & 28.7 & 27.4 & 48.6 & 25.5  \\    
    \pomicon GPPoM2-S Hyb.  & {3.31} & 29.0  & \textbf{33.8} & \textbf{51.9} & \textbf{25.6} \\
    \bottomrule
    \end{tabular}}
    \caption{{\bf Performance.} Comparison of GPT2 trained on 15B tokens from FineWeb and its GPPoM2 version.}
    \label{tab:nanogpt2:perf}
    \end{subtable}
\\
\centering
\begin{subtable}{\linewidth}
\centering\footnotesize
\begin{tabular}{lrrrr}
\toprule
\textbf{Sequence length} & 1k & 4k & 16k & 32k \\
\midrule
GPT2-S MHA & 572.3 & 294.6 & 99.3 & 52.4 \\
\greyrule
\pomicon GPPoM2-S $k=2$ & 995.7 & 893.1 & 624.2 & 447.5 \\
\phantom{\pomicon GPPoM2-S} $k=3$ & 993.5 & 890.5 & 622.79 & 445.78 \\
\phantom{\pomicon GPPoM2-S} $k=5$ & 995.6 & 890.2 & 617.10 & 439.51 \\
\pomicon GPPoM2-S Hyb. & 796.1 & 533.4 & 227.8 & 128.7 \\
\greyrule
Mamba & 56.4 & 56.1 & 56.8 & 56.7 \\
\bottomrule
\end{tabular}
%}
\caption{\textbf{Efficiency.} Token processing speed (in k-tokens/s) at which models are processing (forward) sequences of varying length.}
\label{tab:speed-nlp}
\end{subtable}
\end{tabular}
%}
\vspace{-2mm}
    \caption{{\bf Natural Language Processing.}}
    \label{tab:nanogpt2}
\end{table}

%============================================
\subsection{Natural Language Processing}
\label{sec:nlp}
%============================================
%

\paragraph{Setting.}
We train 125M GPT-2 models on 15B tokens from FineWeb~\cite{penedo2024fineweb}, following the nanoGPT setup~\cite{Karpathy2022}. 
We evaluate models using validation loss, accuracy on the following benchmarks: ARC-E \cite{clark2018think}, HellaSwag~\cite{zellers2019hellaswag}, Winogrande \cite{sakaguchi2021winogrande}, and MMLU \cite{hendrycksmeasuring}. We measure generation speed measured as the time to produce the last token in sequences of varying lengths.  

We consider two configurations:  
(i) \textbf{GPPoM2}, where every self-attention block is replaced by a Polymorpher, and  
(ii) \textbf{GPPoM2 Hyb.}, where pairs of self-attention blocks are replaced by one Polymorpher and one local-attention block with a causal context limited to the last 128 tokens.  
This hybrid design combines the global receptive field of PoM’s state representation with the local, fine-grained modeling of attention; this is conceptually similar to how LSTMs blend long- and short-range dependencies \cite{hochreiter1997long}.  
Moreover, the hybrid architecture preserves a constant computational cost during autoregressive inference, since the local attention context is fixed.

\paragraph{Results.}
As reported in \cref{tab:nanogpt2:perf}, GPPoM2 achieves performance comparable, though slightly below, that of full-attention Transformers. 
It is worth noting that all training hyperparameters were tuned specifically for standard Transformers and we did not change the training recipe.  
In contrast, the Hybrid GPPoM matches, and occasionally surpasses, the performance of GPT2-S, including the checkpoint released by OpenAI that was trained on an order of magnitude more data.  
Both GPPoM variants are substantially mode efficient than Transformers, owing to their linear-time complexity. We show in Table~\ref{tab:speed-nlp} the average speed for a forward pass on a fix token budget (from $128$ batches of $1024$ tokens to $4$ batches of $32$k tokens), using custom triton kernels.
PoM-based models are even substancially faster than Mamba, which only beats attention at very long sequence lengths due to the optimization in Flash-Attention. The degree $k$ does not have impact on the speed.

%============================================
\subsection{OCR}
\label{sec:ocr}
%============================================
\begin{table}[t]\
    \centering
    \resizebox{\columnwidth}{!}{
    \begin{tabular}{l@{\;}l@{\;}c c@{\;}c c} 
    \toprule
          &&
         \multirow{1}{*}{\textbf{\#Params}} &
         \multicolumn{2}{c}{\bf Performance} &
         {\textbf{Throughput}} \\
         \cmidrule(lr){3-3}\cmidrule(lr){4-5}\cmidrule(lr){6-6}
         &\multirow{1}{*}{\textbf{Model}} &  $\times 10^6$&\textbf{CER}{\color{gray!80}$\downarrow$} &\textbf{WER}{\color{gray!80}$\downarrow$} & lines/s \\
    \midrule
      \multirow{3}{*}{\colorbox{Orchid!20}{\scriptsize \rotatebox{90}{LAM}}}
      %&HTR-VT & 53.5 & \textbf{2.8} & \textbf{7.4} & \textbf{607.6}\\ % no FA
      &HTR-VT & 53.5 & \textbf{2.8} & \textbf{7.4} & 620.2\\
      % &\pomicon HTR-VPoM & 53.7 &  3.2 & 9.0 & 586.8 \\ % native pytorch
      &\pomicon HTR-VPoM & 53.7 &  3.2 & 9.0 & \textbf{623.3} \\ % triton
      %&\pomicon  HTR-VPoM Hyb.& 53.5 & \textbf{2.8} & 7.5 & 597.1 \\ % native pytorch
      &\pomicon  HTR-VPoM Hyb.& 53.5 & \textbf{2.8} & 7.5 & 622.7 \\
      \greyrule
      \multirow{3}{*}{\colorbox{YellowOrange!20}{\scriptsize \raisebox{-2mm}{\rotatebox{90}{LAM-Mul.}}}}
      %&HTR-VT & 54.9 & \textbf{3.3}  & \textbf{9.3} & 566.4 \\ no FA
      &HTR-VT & 54.9 & \textbf{3.3}  & \textbf{9.3} & 598.4 \\ % FA
      %&\pomicon  HTR-VPoM & 55.2  & 3.7  & 10.9 & \textbf{645.2} \\ native pytorch
      &\pomicon  HTR-VPoM & 55.2  & 3.7  & 10.9 & \textbf{675.5} \\ % triton
      %&\pomicon  HTR-VPoM Hyb. & 55.1 & \textbf{3.3} & \textbf{9.3} & 600.1  \\ native pytorch
      &\pomicon  HTR-VPoM Hyb. & 55.1 & \textbf{3.3} & \textbf{9.3} & 636.6 \\ % triton
    \bottomrule
    \end{tabular}
    }
    \caption{{\bf Optical Character Recognition.} We compare HTR-VT and HTR-VPoM on LAM single line and LAM multiline dataset. We report the Character Error Rate (CER), Word Error Rate (WER), and throughput (lines/s). The multiline experiments are trained for half the steps.}
    \label{tab:ocr_tab_main}
\end{table}

\paragraph{Setting.} We assess the effectiveness of PoM on a optical character recognition task (OCR). We use the SOTA HTR-VT~\cite{li2025htr}, a ViT based encoder only network, as our baseline. We conduct our experiments on the Ludovico Antonio Muratori (LAM)~\cite{cascianelli2022lam} dataset which consists of 25,823 lines in total. To showcase the benefits of PoM over standard attention we conduct two experiments, specifically, (i) Single line prediction and (ii) Multiple line prediction. Given the HTR-VT network only has four attention blocks we experiment with two type of PoM architecture: (i) HTR-VPoM: replacing all the attention layers from HTR-VT to PoM layers and (ii) HTR-VPoM Hybrid: alternating attention layers and PoM layers that perform full attention (no causal mask).

\paragraph{Results.} We report the scores in \cref{tab:ocr_tab_main}. For the single-line prediction experiments, PoM based architectures are on-par with attention regarding CER (character error rate) and WER (word error rate). However, given the sequence length is just 128 tokens, attention benefits from the low occupancy level of the GPU to compensate for its quadratic cost. 
To take advantage of the linear cost of PoM, we also consider predicting entire pages at a time, consisting of multiple lines.
We concatenate lines from the same page to produce a wide image, similar to unrolling the page horizontally. This adds up to 16 lines, or about 2k tokens on average.
On this multi-line setup, the performances of HTR-VT are slightly degraded (3.3 CER \textit{vs} 2.8 originally) which may come from underfitting a larger input signal. HTR-VPoM Hybrid is on par with the original architecture, while increasing throughput by about 6\%. HTR-VPoM obtains 0.4 worst CER but achieves a significant 12\% speedup compared to Flash-Attention. Interestingly, the throughput of PoM models is higher in the multi-line setup than in the single line setup\footnote{We assume this is due to the GPU being under-occupied and the latency having a bigger impact.}, the opposite of attention models. This demonstrates the superior scalability of PoM compared to attention, even with Flash-Attention.

%========================================
\subsection{Earth Observation Analysis}
\label{sec:eo}
%========================================
\paragraph{Setting.}
We evaluate PoM on a standard geospatial task: crop type mapping from satellite multi-spectral image time series. 
Experiments are conducted on fold~1 of the PASTIS dataset~\cite{garnot2021panoptic}, which contains Sentinel-2 time series with 10 spectral bands at a spatial resolution of $10$\,m per pixel. 
Each pixel is annotated with one of 19 crop type classes. 
We report overall accuracy (OA), mean Intersection over Union (mIoU), FLOPs, and throughput.  

Our baseline is the TSViT model~\cite{Tarasiou_2023_CVPR}, a state-of-the-art architecture for supervised crop segmentation that alternates between temporal and spatial self-attention layers. 
We construct \textbf{TSViPoM} by replacing all Transformer blocks in TSViT with Polymorpher blocks, keeping the rest of the architecture unchanged.

\paragraph{Results.}
TSViPoM achieves performance comparable to the TSViT baseline while offering a substantial speedup---over 300\% faster inference---demonstrating the efficiency of PoM as a drop-in replacement for attention in spatio-temporal modeling. 
This improvement is particularly relevant for geospatial applications, where inputs often span long temporal sequences, large spatial extents, and multiple spectral bands. 
In such settings, the quadratic scaling of self-attention becomes a major bottleneck, forcing practitioners to crop scenes, reduce temporal resolution, or downsample spectral information. 
By contrast, PoM’s linear complexity allows the model to process long and high-dimensional sequences directly, preserving temporal continuity and spectral richness while maintaining fast and memory-efficient inference.

\begin{table}\small
    \centering
    \begin{tabular}{l@{\;}c c@{\;}c c} 
    \toprule
          &
         \multirow{1}{*}{\textbf{\#Params}} &
         \multicolumn{2}{c}{\bf Performance} &
         {\textbf{Throughput}} \\ \cmidrule(lr){2-2}\cmidrule(lr){3-4}\cmidrule(lr){5-5}
         \multirow{1}{*}{\textbf{Model}} &  $\times 10^6$&\textbf{OA}{\color{gray!80}$\uparrow$} &\textbf{mIoU}{\color{gray!80}$\uparrow$} & km\textsuperscript{2}/s \\
    \midrule
         \textcolor{gray!90}{TSViT (paper)} &\color{gray!90} 1.7 & \color{gray!90} 83.4 & \color{gray!90} 65.1 & \color{gray!90} - \\
         TSViT (retrained) & 1.7 & \bf 82.8 & \bf 64.6 & 7.2 \\\greyrule
         \pomicon TSViPoM  & 1.7 & 82.7 & \bf 64.6 & \bf 26.3 \\
    \bottomrule
    \end{tabular}
    \caption{{\bf Satellite Image Time Series Segmentation.} We compare of TSViT and TSViPoM on PASTIS fold 1. We report the overall accuracy (OA), mean IoU (mIoU), and throughput.}
    \label{tab:geo_results}
\end{table}

%============================================
\subsection{Point Cloud Segmentation}
\label{sec:3d}
%============================================
\begin{table}\small
    \centering
    \setlength{\tabcolsep}{4.5pt}
    \begin{tabular}{l@{\,}l c c c}
        \toprule
         && 
         \textbf{\#Param} &
         \textbf{Performance} &
         \textbf{Throughput} \\
        \cmidrule(lr){3-3}\cmidrule(lr){4-4}\cmidrule(lr){5-5}
        &\textbf{Model} & $\times 10^6$ & 
        \textbf{mIoU}{\color{gray!80}$\uparrow$} &
        $10^6$ pts/s
        \\\midrule
        \multirow{3}{*}{\colorbox{Orchid!20}{\scriptsize \rotatebox{90}{Scannet}}} & PTv3 & 46.2 & 76.7 & 1.98
        \\
        & \pomicon PPoMv3 & 47.9 & 75.9 & \bf 2.17 \\
        &\pomicon PPoMv3 Hyb. & 46.8 &  \bf 76.8 & 2.01
        \\\greyrule 
        \multirow{3}{*}{\colorbox{YellowOrange!20}{\scriptsize \rotatebox{90}{SKITTI}}}&PTv3  & 46.2 &  67.2& 1.12
        \\
        & \pomicon PPoMv3 &  47.9 &  66.6& \bf 1.22\\
        & \pomicon PPoMv3 Hyb. &  46.8 &  \bf 67.5 &  1.15
        \\\bottomrule
    \end{tabular}
    \caption{{\bf 3D Point Cloud Segmentation.} We compare PTv3 with PPoMv3 on two classic 3D datasets.
    }
    \label{tab:3d}
\end{table}

\paragraph{Setting.}
We evaluate PoM for 3D point cloud semantic segmentation on two standard benchmarks: indoor scene understanding with ScanNet~\cite{scannet} and autonomous driving with SemanticKITTI~\cite{semantickitti}. 
Our baseline is the state-of-the-art PointTransformerV3 (PTv3)~\cite{ptv3}, which serializes point clouds along space-filling curves that preserve local geometric continuity. 
These serialized sequences are hierarchically pooled to form a U-Net-like encoder–decoder architecture. 
At each scale, PTv3 applies several multi-head attention (MHA) layers with a local context of 1024 points, a limit imposed by the quadratic cost of attention and GPU memory constraints.

We investigate two PoM-based replacements for the attention modules:
(i) \textbf{PPoMv3}, which replaces all MHA layers with PoM while keeping the same context length, and  
(ii) \textbf{PPoMv3-Hybrid}, which alternates MHA layers, and PoM layers with a context window of 4096 points in the two highest-resolution stages of the U-Net.

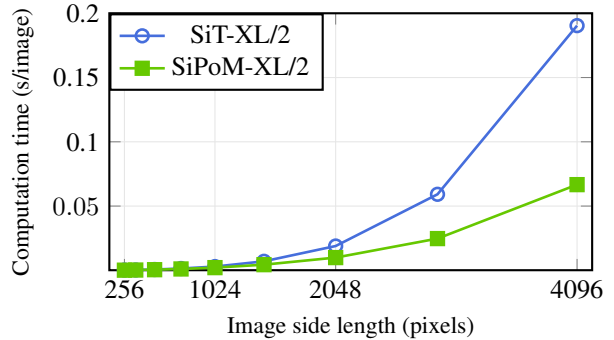
\begin{figure}[!t]
    \centering
    \begin{tikzpicture}
\begin{axis}[
    width=8cm, height=5cm,
    xlabel={\small Image side length (pixels)},
    ylabel={\small Computation time  (s/image)},
    %ymode=log,
    xtick={256,1024,2048,4096},
    xticklabels={256, 1024, 2048, 4096},
    ytick={0.05,0.1,0.15,0.2},
    yticklabels={$0.05$,0.1,0.15,0.2},
    ymin=0.0001, ymax=0.2,
    xmin=128, xmax=4224,
    grid=both,
    grid style={gray!20},
    legend style={at={(0.0,1.0)}, anchor=north west, legend columns=1},
    thick,
    mark options={scale=1.2},
]

% GPT2-L
\addplot[
    color=colTrans,
    mark=o,
    line width=1pt
] coordinates {
(256, {0.0002474415141227837})
(352, {0.00029000554405797627})
(512, {0.0006429552278996198})
(736, {0.0014204099743528786})
(1024, {0.0029289669421632427})
(1440, {0.007026392328025394})
(2048, {0.01891895949142054})
(2912, {0.0591703769122978})
(4096, {0.1903850698551105})
};
\addlegendentry{SiT-XL/2}

% GPPoM2-L
\addplot[
    color=colPOM,
    mark=square*,
    line width=1pt
] coordinates {
(256, {0.0001227446490031525})
(352, {0.0002407291617601004})
(512, {0.000504921072729303})
(736, {0.0010763675006273843})
(1024, {0.002105088429989337})
(1440, {0.004406789621807548})
(2048, {0.009923821208758454})
(2912, {0.024707892990409163})
(4096, {0.06667813312014914})
};
\addlegendentry{SiPoM-XL/2}

\end{axis}
\end{tikzpicture}
    % \vspace{-2mm}
    \caption{{\bf Computational Efficiency.} We report the speed at generating image for various resolutions for SiT-XL/2 and its PoM counterpart SiPoM-XL/2.}
    \label{fig:sipom_speed}
\end{figure}

\begin{figure*}[!tb]
    \centering
    \includegraphics[width=\textwidth]{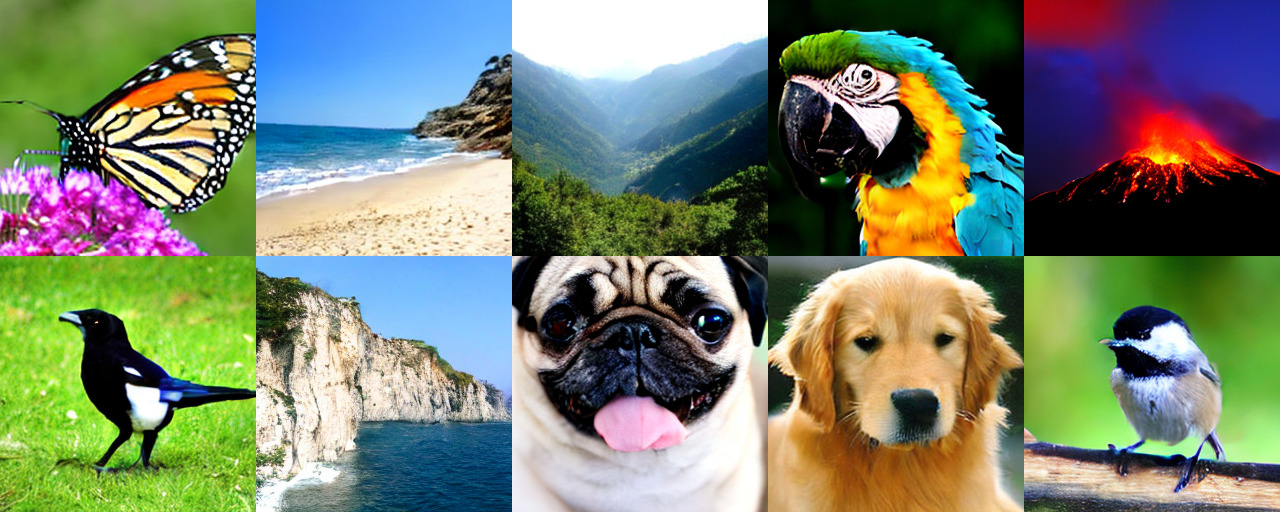}
    \caption{\textbf{Qualitative Results on Class-Conditional Generation}. Images sampled from SiPom-XL/2 with different classes at 256 resolution. We use classifier-free guidance with $\omega=6$.}
    \label{fig:curated}
\end{figure*}

\paragraph{Results.}
Results are reported in~\cref{tab:3d}. 
The hybrid PPoMv3 matches or slightly surpasses PTv3 on both datasets while reducing computation time by 2\%. 
Pure PoM models are up to 10\% faster than Flash-Attention, with only a minor drop in accuracy. 
A practical advantage of PoM’s linear complexity is that entire point clouds can be processed in memory without chunking. 
While this does not translate into accuracy gains on standard CV benchmarks, it is a significant benefit for large-scale industrial applications where clean spatial partitioning such as single rooms is not applicable.

\begin{table}[tb]
    \centering
    \setlength{\tabcolsep}{5pt}
    \small
    \begin{tabular}{l c c c}
        \toprule
         & 
         \textbf{\#Param} &
         \textbf{Performance} &
         \textbf{Throughput} \\
        \cmidrule(lr){2-2}\cmidrule(lr){3-3}\cmidrule(lr){4-4}
        \textbf{Model} & $\times 10^6$ & 
        \textbf{FID}{\color{gray!80}$\downarrow$} &
        img/s
        \\
        \midrule
        SiT-L/2 &  458 & 18.8 &  \hphantom{1,}9842  \\
        \pomicon SiPoM-L/2 &   414 &  \textbf{17.6} &  \bf 11,954 \\
        \greyrule
        SiT-XL/2 & 675 & \textbf{17.2} &  \hphantom{1,}4041 \\
        \pomicon SiPoM-XL/2 & 609 &  \textbf{17.2} & \bf\hphantom{1,}8146  \\
    \bottomrule
    \end{tabular}
    \caption{\textbf{Class-Conditional Image Generation}. We compare SiT~\cite{ma24eccv} with SiPoM trained on ImageNet for generating images of resolution 256$^2$ for a batch size of 1024.}
    \label{tab:img_sota}
\end{table}
%============================================
\subsection{Class-Conditional Image Generation}
\label{sec:imgen}
%============================================
\paragraph{Setting.}
We evaluate PoM on a class-conditional image generation task. For that matter, we take the popular SiT architecture~\cite{ma24eccv} and replace all Attention blocks with PoM blocks. We use the AdaLNZero strategy to condition the network on the class.
We consider 2 sizes of models, namely SiT-L and SiT-XL with a patch-size of 2, leading to our PoM variant named SiPoM-L/2 and SiPoM-XL/2 respectively.
The training is performed on ImageNet at 256 resolution, with the same training parameters as in the original paper.
The performances are evaluated with FID as is standard for this task. The efficiency is measured as throughput, which is the number of images per second that can be processed by a forward pass of the network for a batch size of 1024 images.

\paragraph{Quantitative results}

We compare the results of our SiPoM models to the original scores from \cite{ma24eccv} on Table~\ref{tab:img_sota}.
Our SiPoM variants achieve the same visual fidelity as the original attention-based models, as measured by the FID.
We also measure the throughput of the SiT-XL/2 and the SiPoM-XL/2 models at various resolution, which we report in Figure~\ref{fig:sipom_speed}.
As excepted, SiPoM becomes much more efficient than its attention-based counterpart at higher resolutions, starting at 1024 pixels per side. This is easily explain by the number of tokens growing quadratically with the resolution of the images, leading to longer sequences which further accentuate the gap between linear cost of PoM and the quadratic cost of attention.

\paragraph{Qualitative results}
We show images sampled from SiPoM-XL/2 in Figure~\ref{fig:curated}. The images are sampled using Classifier Free Guidance with $\omega=6$.
The visual quality of the images is on par with equivalent models from the literature.
This validates that SiPoM-XL/2 is able to produce the same quality of pictures as its attention counterpart.

\section{Conclusion}
We introduced \textbf{PoM}, the \emph{Polynomial Mixer}, a neural network building block that serves as a drop-in, efficient replacement for self-attention in Transformer architectures. 
PoM combines \emph{linear complexity} with the \emph{universal sequence-to-sequence approximator}.  
We validated PoM across five diverse domains—natural language processing, Earth observation, 3D point cloud segmentation, optical character recognition, and image generation. We demonstrate that hybrid PoM/self-attention architectures consistently match or surpass the performance of full-attention models while offering superior speed and scalability.  

These properties make PoM a compelling choice for long-sequence applications such as high-definition video generation or multimodal large language models, which could greatly benefit from its $\mathcal{O}(1)$ inference under causal masking. 
While large-scale experiments remain beyond our current computational scope, PoM lays the groundwork for the next generation of efficient, general-purpose, attention-free sequence models.

\section{Acknowledgment}
This work was supported by ANR project sharp ANR-23-PEIA-0008 in
the context of the PEPR IA. This project was provided with computing and storage resources by GENCI at IDRIS thanks to the grants 2025-A0191013085 and 2025-A0181016234 on the supercomputer Jean Zay's H100 partition.
The authors would like to thank Vincent Lepetit,
Gül Varol and Dimitris Samaras for their insightful comments and suggestions.

{
    \small
    \bibliographystyle{ieeenat_fullname}
    {\balance
    {
    \bibliography{pom}

@String(PAMI = {IEEE Trans. Pattern Anal. Mach. Intell.})

@String(CVPR= {IEEE Conf. Comput. Vis. Pattern Recog.})

@String(ICCV= {Int. Conf. Comput. Vis.})

@String(ECCV= {Eur. Conf. Comput. Vis.})

@String(NIPS= {Adv. Neural Inform. Process. Syst.})

@String(ICPR = {Int. Conf. Pattern Recog.})

@String(ICIP = {IEEE Int. Conf. Image Process.})

@String(ICLR = {Int. Conf. Learn. Represent.})

@String(ICML = {Int. Conf. Mach. Learn.})

@String(AAAI = {AAAI})

@String(PAMI  = {IEEE TPAMI})

@String(CVPR  = {CVPR})

@String(ICCV  = {ICCV})

@String(ECCV  = {ECCV})

@String(NIPS  = {NeurIPS})

@String(ICPR  = {ICPR})

@String(ICIP  = {ICIP})

@String(ICLR  = {ICLR})

@inproceedings{picard2011improving,
  title={Improving image similarity with vectors of locally aggregated tensors},
  author={Picard, David and Gosselin, Philippe-Henri},
  booktitle={ICIP},
  pages={669--672},
  year={2011}
}

@inproceedings{jacob2019metric,
  title={Metric learning with horde: High-order regularizer for deep embeddings},
  author={Jacob, Pierre and Picard, David and Histace, Aymeric and Klein, Edouard},
  booktitle=ICCV,
  pages={6539--6548},
  year={2019}
}

@inproceedings{esser24icml,
  title={Scaling rectified flow transformers for high-resolution image synthesis},
  author={Esser, Patrick and Kulal, Sumith and Blattmann, Andreas and Entezari, Rahim and M{\"u}ller, Jonas and Saini, Harry and Levi, Yam and Lorenz, Dominik and Sauer, Axel and Boesel, Frederic and others},
  booktitle=ICML,
year=2024
}

@inproceedings{vaswani17nips,
 author = {Vaswani, Ashish and Shazeer, Noam and Parmar, Niki and Uszkoreit, Jakob and Jones, Llion and Gomez, Aidan N and Kaiser, {\L}ukasz and Polosukhin, Illia},
 booktitle = NIPS,
 title = {Attention is All you Need},
 year = {2017}
}

@article{kaplan2020scaling,
  title={Scaling laws for neural language models},
  author={Kaplan, Jared and McCandlish, Sam and Henighan, Tom and Brown, Tom B and Chess, Benjamin and Child, Rewon and Gray, Scott and Radford, Alec and Wu, Jeffrey and Amodei, Dario},
  journal={arXiv preprint arXiv:2001.08361},
  year={2020}
}

@inproceedings{peebles23iccv,
  title={Scalable diffusion models with transformers},
  author={Peebles, William and Xie, Saining},
  booktitle=ICCV,
  year={2023}
}

@inproceedings{ma24eccv,
  title={{SIT}: {E}xploring flow and diffusion-based generative models with scalable interpolant transformers},
  author={Ma, Nanye and Goldstein, Mark and Albergo, Michael S and Boffi, Nicholas M and Vanden-Eijnden, Eric and Xie, Saining},
  booktitle=ECCV,
  year={2024}
}

@inproceedings{zhai22cvpr,
  title={Scaling vision transformers},
  author={Zhai, Xiaohua and Kolesnikov, Alexander and Houlsby, Neil and Beyer, Lucas},
  booktitle=CVPR,
  year={2022}
}

@inproceedings{yan24cvpr,
  title={Diffusion models without attention},
  author={Yan, Jing Nathan and Gu, Jiatao and Rush, Alexander M},
  booktitle=CVPR,
  year={2024}
}

@InProceedings{hu24eccv,
      title={{ZigMa}: A {DiT}-style Zigzag {Mamba} Diffusion Model},
      author={Vincent Tao Hu and Stefan Andreas Baumann and Ming Gui and Olga Grebenkova and Pingchuan Ma and Johannes Fischer and Björn Ommer},
      booktitle = ECCV,
      year={2024}
}

@inproceedings{gu21iclr,
  title={Efficiently Modeling Long Sequences with Structured State Spaces},
  author={Gu, Albert and Goel, Karan and Re, Christopher},
  booktitle=ICLR,
  year={2021}
}

@inproceedings{gu21nips,
  title={Combining recurrent, convolutional, and continuous-time models with linear state space layers},
  author={Gu, Albert and Johnson, Isys and Goel, Karan and Saab, Khaled and Dao, Tri and Rudra, Atri and R{\'e}, Christopher},
  booktitle=NIPS,
  year={2021}
}

@inproceedings{
Yun20ICLR,
title={Are Transformers universal approximators of sequence-to-sequence functions?},
author={Chulhee Yun and Srinadh Bhojanapalli and Ankit Singh Rawat and Sashank Reddi and Sanjiv Kumar},
booktitle=ICLR,
year={2020}
}

@inproceedings{ho20nips,
  title={Denoising diffusion probabilistic models},
  author={Ho, Jonathan and Jain, Ajay and Abbeel, Pieter},
  booktitle=NIPS,
  year={2020}
}

@inproceedings{nichol21icml,
  title={Improved denoising diffusion probabilistic models},
  author={Nichol, Alexander Quinn and Dhariwal, Prafulla},
  booktitle=ICML,
  year={2021}
}

@inproceedings{song21iclr,
  title={Score-Based Generative Modeling through Stochastic Differential Equations},
  author={Song, Yang and Sohl-Dickstein, Jascha and Kingma, Diederik P and Kumar, Abhishek and Ermon, Stefano and Poole, Ben},
  booktitle=ICLR,
  year={2021}
}

@article{balaji22eDiffI,
    title={{eDiff-I: T}ext-to-Image Diffusion Models with Ensemble of Expert Denoisers},
    author={Yogesh Balaji and Seungjun Nah and Xun Huang and Arash Vahdat and Jiaming Song and Qinsheng Zhang and Karsten Kreis and Miika Aittala and Timo Aila and Samuli Laine and Bryan Catanzaro and Tero Karras and Ming-Yu Liu},
    journal={arXiv:2211.01324},
    year={2022}
}

@inproceedings{karras24cvpr,
  title={Analyzing and improving the training dynamics of diffusion models},
  author={Karras, Tero and Aittala, Miika and Lehtinen, Jaakko and Hellsten, Janne and Aila, Timo and Laine, Samuli},
  booktitle=CVPR,
  year={2024}
}

@article{hang2024improved,
  title={Improved noise schedule for diffusion training},
  author={Hang, Tiankai and Gu, Shuyang},
  journal=ICCV,
  year={2024}
}

@inproceedings{lipman22iclr,
  title={Flow Matching for Generative Modeling},
  author={Lipman, Yaron and Chen, Ricky TQ and Ben-Hamu, Heli and Nickel, Maximilian and Le, Matthew},
  booktitle=ICLR,
  year={2022}
}

@inproceedings{shi24nips,
  title={Diffusion {S}chr{\"o}dinger bridge matching},
  author={Shi, Yuyang and De Bortoli, Valentin and Campbell, Andrew and Doucet, Arnaud},
  booktitle=NIPS,
  year={2024}
}

@inproceedings{liu23iclr,
  title={Flow Straight and Fast: {L}earning to Generate and Transfer Data with Rectified Flow},
  author={Liu, Xingchao and Gong, Chengyue and others},
  booktitle=ICLR,
  year={2023}
}

@inproceedings{rombach22cvpr,
  title={High-resolution image synthesis with latent diffusion models},
  author={Rombach, Robin and Blattmann, Andreas and Lorenz, Dominik and Esser, Patrick and Ommer, Bj{\"o}rn},
  booktitle=CVPR,
  year={2022}
}

@inproceedings{jaegle22iclr,
  title={Perceiver {IO}: {A} General Architecture for Structured Inputs \& Outputs},
  author={Jaegle, Andrew and Borgeaud, Sebastian and Alayrac, Jean-Baptiste and Doersch, Carl and Ionescu, Catalin and Ding, David and Koppula, Skanda and Zoran, Daniel and Brock, Andrew and Shelhamer, Evan and others},
  booktitle=ICLR,
  year={2022}
}

@inproceedings{jabri23icml,
  title={Scalable adaptive computation for iterative generation},
  author={Jabri, Allan and Fleet, David J and Chen, Ting},
  booktitle=ICML,
  year={2023}
}

@inproceedings{dufour24cvpr,
  title={Don't drop your samples! {C}oherence-aware training benefits Conditional diffusion},
  author={Dufour, Nicolas and Besnier, Victor and Kalogeiton, Vicky and Picard, David},
  booktitle=CVPR,
  year={2024}
}

@article{yang2024cogvideox,
  title={Cog{V}ideo{X}: {T}ext-to-video diffusion models with an expert transformer},
  author={Yang, Zhuoyi and Teng, Jiayan and Zheng, Wendi and Ding, Ming and Huang, Shiyu and Xu, Jiazheng and Yang, Yuanming and Hong, Wenyi and Zhang, Xiaohan and Feng, Guanyu and others},
  journal=ICLR,
  year={2025}
}

@inproceedings{kitaev2020reformer,
    title       = {Re{F}ormer: {T}he Efficient Transformer},
    author      = {Nikita Kitaev and Lukasz Kaiser and Anselm Levskaya},
    booktitle   = ICLR,
    year        = {2020}
}

@article{wang2020linformer,
  title={Lin{F}ormer: {S}elf-attention with linear complexity},
  author={Wang, Sinong and Li, Belinda Z and Khabsa, Madian and Fang, Han and Ma, Hao},
  journal={arXiv:2006.04768},
  year={2020}
}

@article{child2019generating,
  title={Generating long sequences with sparse transformers},
  author={Child, Rewon and Gray, Scott and Radford, Alec and Sutskever, Ilya},
  journal={arXiv:1904.10509},
  year={2019}
}

@article{achiam2023gpt,
  title={{GPT-4} technical report},
  author={Achiam, Josh and Adler, Steven and Agarwal, Sandhini and Ahmad, Lama and Akkaya, Ilge and Aleman, Florencia Leoni and Almeida, Diogo and Altenschmidt, Janko and Altman, Sam and Anadkat, Shyamal and others},
  journal={arXiv:2303.08774},
  year={2023}
}

@article{team2023gemini,
  title={{GEMINI}: {A} family of highly capable multimodal models},
  author={Team, Gemini and Anil, Rohan and Borgeaud, Sebastian and Alayrac, Jean-Baptiste and Yu, Jiahui and Soricut, Radu and Schalkwyk, Johan and Dai, Andrew M and Hauth, Anja and Millican, Katie and others},
  journal={arXiv:2312.11805},
  year={2023}
}

@article{dubey2024llama,
  title={The {LLAMA} 3 herd of models},
  author={Dubey, Abhimanyu and Jauhri, Abhinav and Pandey, Abhinav and Kadian, Abhishek and Al-Dahle, Ahmad and Letman, Aiesha and Mathur, Akhil and Schelten, Alan and Yang, Amy and Fan, Angela and others},
  journal={arXiv:2407.21783},
  year={2024}
}

@inproceedings{dao22nips,
  title={Flash{A}ttention: {F}ast and memory-efficient exact attention with io-awareness},
  author={Dao, Tri and Fu, Dan and Ermon, Stefano and Rudra, Atri and R{\'e}, Christopher},
  booktitle=NIPS,
  year={2022}
}

@article{dao2023flashattention,
  title={Flash{A}ttention-2: {F}aster attention with better parallelism and work partitioning},
  author={Dao, Tri},
  journal=ICLR,
  year={2024}
}

@article{brandon2024reducing,
  title={Reducing Transformer Key-Value Cache Size with Cross-Layer Attention},
  author={Brandon, William and Mishra, Mayank and Nrusimha, Aniruddha and Panda, Rameswar and Kelly, Jonathan Ragan},
  journal=NIPS,
  year={2024}
}

@article{luohe2024keep,
  title={Keep the Cost Down: {A} Review on Methods to Optimize {LLM}'s {KV}-Cache Consumption},
  author={Luohe, Shi and Hongyi, Zhang and Yao, Yao and Zuchao, Li and Hai, Zhao},
  journal={COLM},
  year={2024}
}

@inproceedings{tolstikhin21nips,
  title={{MLP}-{M}ixer: {A}n all-{MLP} architecture for vision},
  author={Tolstikhin, Ilya O and Houlsby, Neil and Kolesnikov, Alexander and Beyer, Lucas and Zhai, Xiaohua and Unterthiner, Thomas and Yung, Jessica and Steiner, Andreas and Keysers, Daniel and Uszkoreit, Jakob and others},
  booktitle=NIPS,
  year={2021}
}

@article{touvron2022resmlp,
  title={Res{MLP}: {F}eedforward networks for image classification with data-efficient training},
  author={Touvron, Hugo and Bojanowski, Piotr and Caron, Mathilde and Cord, Matthieu and El-Nouby, Alaaeldin and Grave, Edouard and Izacard, Gautier and Joulin, Armand and Synnaeve, Gabriel and Verbeek, Jakob and others},
  journal=PAMI,
  year={2022}
}

@article{zuo2024falcon,
  title={Falcon {M}amba: {T}he First Competitive Attention-free {7B} Language Model},
  author={Zuo, Jingwei and Velikanov, Maksim and Rhaiem, Dhia Eddine and Chahed, Ilyas and Belkada, Younes and Kunsch, Guillaume and Hacid, Hakim},
  journal={arXiv:2410.05355},
  year={2024}
}

@inproceedings{dao24icml,
  title={Transformers are {SSMs}: {G}eneralized Models and Efficient Algorithms Through Structured State Space Duality},
  author={Dao, Tri and Gu, Albert},
  booktitle=ICML,
year={2024}
}

@article{lieber2024jamba,
  title={Jamba: {A} hybrid transformer-{M}amba language model},
  author={Lieber, Opher and Lenz, Barak and Bata, Hofit and Cohen, Gal and Osin, Jhonathan and Dalmedigos, Itay and Safahi, Erez and Meirom, Shaked and Belinkov, Yonatan and Shalev-Shwartz, Shai and others},
  journal=ICLR,
  year={2025}
}

@article{glorioso2024zamba,
  title={Zamba: {A} Compact {7B} {SSM} Hybrid Model},
  author={Glorioso, Paolo and Anthony, Quentin and Tokpanov, Yury and Whittington, James and Pilault, Jonathan and Ibrahim, Adam and Millidge, Beren},
  journal={arXiv:2405.16712},
  year={2024}
}

@article{zhu2024vision,
  title={Vision {M}amba: {E}fficient visual representation learning with bidirectional state space model},
  author={Zhu, Lianghui and Liao, Bencheng and Zhang, Qian and Wang, Xinlong and Liu, Wenyu and Wang, Xinggang},
  journal={arXiv:2401.09417},
  year={2024}
}

@misc{liu2024vmambavisualstatespace,
      title={VMamba: {V}isual State Space Model}, 
      author={Yue Liu and Yunjie Tian and Yuzhong Zhao and Hongtian Yu and Lingxi Xie and Yaowei Wang and Qixiang Ye and Yunfan Liu},
      journal=NIPS,
      year={2024}
}

@misc{pei2024efficientvmambaatrousselectivescan,
      title={Efficient{VM}amba: {A}trous Selective Scan for Light Weight Visual Mamba}, 
      author={Xiaohuan Pei and Tao Huang and Chang Xu},
      journal=AAAI,
      year={2025}
}

@article{gao2024lumina,
  title={{Lumina-T2X: T}ransforming Text into Any Modality, Resolution, and Duration via Flow-based Large Diffusion Transformers},
  author={Gao, Peng and Zhuo, Le and Lin, Ziyi and Liu, Chris and Chen, Junsong and Du, Ruoyi and Xie, Enze and Luo, Xu and Qiu, Longtian and Zhang, Yuhang and others},
  journal={arXiv:2405.05945},
  year={2024}
}

@InProceedings{Liu_2024_CVPR,
    author    = {Liu, Yujian and Zhang, Yang and Jaakkola, Tommi and Chang, Shiyu},
    title     = {Correcting Diffusion Generation through Resampling},
    booktitle = CVPR,
    year      = {2024}
}

@InProceedings{Gokaslan_2024_CVPR,
    author    = {Gokaslan, Aaron and Cooper, A. Feder and Collins, Jasmine and Seguin, Landan and Jacobson, Austin and Patel, Mihir and Frankle, Jonathan and Stephenson, Cory and Kuleshov, Volodymyr},
    title     = {Common{C}anvas: {O}pen Diffusion Models Trained on Creative-Commons Images},
    booktitle = CVPR,
    year      = {2024}
}

@InProceedings{Wallace_2024_CVPR,
    author    = {Wallace, Bram and Dang, Meihua and Rafailov, Rafael and Zhou, Linqi and Lou, Aaron and Purushwalkam, Senthil and Ermon, Stefano and Xiong, Caiming and Joty, Shafiq and Naik, Nikhil},
    title     = {Diffusion Model Alignment Using Direct Preference Optimization},
    booktitle = CVPR,
    year      = {2024}
}

@InProceedings{Si_2024_CVPR,
    author    = {Si, Chenyang and Huang, Ziqi and Jiang, Yuming and Liu, Ziwei},
    title     = {Free{U}: {F}ree Lunch in Diffusion {U}-Net},
    booktitle = CVPR,
    year      = {2024}
}

@InProceedings{Karras_2024_CVPR,
    author    = {Karras, Tero and Aittala, Miika and Lehtinen, Jaakko and Hellsten, Janne and Aila, Timo and Laine, Samuli},
    title     = {Analyzing and Improving the Training Dynamics of Diffusion Models},
    booktitle = CVPR,
    year      = {2024}
}

@inproceedings{hatamizadeh2025diffit,
  title={Dif{F}it: {D}iffusion vision transformers for image generation},
  author={Hatamizadeh, Ali and Song, Jiaming and Liu, Guilin and Kautz, Jan and Vahdat, Arash},
  booktitle=ECCV,
  year={2024}
}

@inproceedings{wei2024powerful,
  title={Powerful and Flexible: {P}ersonalized Text-to-Image Generation via Reinforcement Learning},
  author={Wei, Fanyue and Zeng, Wei and Li, Zhenyang and Yin, Dawei and Duan, Lixin and Li, Wen},
  booktitle=ECCV,
  year={2024}
}

@inproceedings{chen2024pixart,
  title={Pixart-$\backslash$sigma: {W}eak-to-strong training of diffusion transformer for 4k text-to-image generation},
  author={Chen, Junsong and Ge, Chongjian and Xie, Enze and Wu, Yue and Yao, Lewei and Ren, Xiaozhe and Wang, Zhongdao and Luo, Ping and Lu, Huchuan and Li, Zhenguo},
  booktitle=ECCV,
  year={2024}
}

@inproceedings{lee2025parrot,
  title={Parrot: {P}areto-optimal multi-reward reinforcement learning framework for text-to-image generation},
  author={Lee, Seung Hyun and Li, Yinxiao and Ke, Junjie and Yoo, Innfarn and Zhang, Han and Yu, Jiahui and Wang, Qifei and Deng, Fei and Entis, Glenn and He, Junfeng and others},
  booktitle=ECCV,
  year={2025}
}

@InProceedings{Zhou_2024_CVPR,
    author    = {Zhou, Zhenyu and Chen, Defang and Wang, Can and Chen, Chun},
    title     = {Fast ODE-based Sampling for Diffusion Models in Around 5 Steps},
    booktitle = CVPR,
    year      = {2024}
}

@InProceedings{Bai_2024_CVPR,
    author    = {Bai, Xingjian and Melas-Kyriazi, Luke},
    title     = {Fixed Point Diffusion Models},
    booktitle = CVPR,
    year      = {2024}
}

@inproceedings{zhao2023mobilediffusion,
  title={Mobile{D}iffusion: {I}nstant text-to-image generation on mobile devices},
  author={Zhao, Yang and Xu, Yanwu and Xiao, Zhisheng and Jia, Haolin and Hou, Tingbo},
  booktitle=ECCV,
  year={2024}
}

@article{chen2023fit,
  title={{FIT: F}ar-reaching Interleaved Transformers},
  author={Chen, Ting and Li, Lala},
  journal={arXiv:2305.12689},
  year={2023}
}

@inproceedings{crowson2024scalable,
  title={Scalable high-resolution pixel-space image synthesis with hourglass diffusion transformers},
  author={Crowson, Katherine and Baumann, Stefan Andreas and Birch, Alex and Abraham, Tanishq Mathew and Kaplan, Daniel Z and Shippole, Enrico},
  booktitle=ICML,
  year={2024}
}

@inproceedings{gu2023matryoshka,
  title={Matryoshka diffusion models},
  author={Gu, Jiatao and Zhai, Shuangfei and Zhang, Yizhe and Susskind, Joshua M and Jaitly, Navdeep},
  booktitle=ICLR,
  year={2023}
}

@article{garnot2021panoptic,
  title={Panoptic Segmentation of Satellite Image Time Series
with Convolutional Temporal Attention Networks},
  author={Sainte Fare Garnot, Vivien  and Landrieu, Loic },
  journal=ICCV,
  year={2021}
}

@InProceedings{Tarasiou_2023_CVPR,
    author    = {Tarasiou, Michail and Chavez, Erik and Zafeiriou, Stefanos},
    title     = {{ViTs} for {SITS}: {V}ision Transformers for Satellite Image Time Series},
    booktitle = CVPR,
    year      = {2023}
}

@inproceedings{dosovitskiy2020image,
  title={An Image is Worth 16x16 Words: {T}ransformers for Image Recognition at Scale},
  author={Dosovitskiy, Alexey and Beyer, Lucas and Kolesnikov, Alexander and Weissenborn, Dirk and Zhai, Xiaohua and Unterthiner, Thomas and Dehghani, Mostafa and Minderer, Matthias and Heigold, Georg and Gelly, Sylvain and others},
  booktitle=ICLR,
  year={2020}
}

@inproceedings{semantickitti,
  author = {Behley, Jens and Garbade, Martin and Milioto, Andres and Quenzel, Jan and Behnke, Sven and Stachniss, Cyrill and Gall, Jurgen},
  title = {{SemanticKITTI}: {A} Dataset for Semantic Scene Understanding of {LiDAR} Sequences},
  booktitle = ICCV,
  year = 2019
}

@inproceedings{ptv3,
    title={Point {T}ransformer V3: {S}impler, Faster, Stronger},
    author={Wu, Xiaoyang and Jiang, Li and Wang, Peng-Shuai and Liu, Zhijian and Liu, Xihui and Qiao, Yu and Ouyang, Wanli and He, Tong and Zhao, Hengshuang},
    booktitle=CVPR,
    year={2024}
}

@inproceedings{scannet,
  title        = {{ScanNet}: {R}ichly-annotated {3D} Reconstructions of Indoor Scenes},
  author       = {Dai, Angela and Chang, Angel X. and Savva, Manolis and Halber, Maciej and Funkhouser, Thomas and Niessner, Matthias},
  year         = 2017,
  booktitle    = CVPR
}

@article{clark2018think,
  title={Think you have solved question answering? {T}try {ARC}, the {AI2} reasoning challenge},
  author={Clark, Peter and Cowhey, Isaac and Etzioni, Oren and Khot, Tushar and Sabharwal, Ashish and Schoenick, Carissa and Tafjord, Oyvind},
  journal={arXiv:1803.05457},
  year={2018}
}

@inproceedings{hendrycksmeasuring,
  title={Measuring Massive Multitask Language Understanding},
  author={Hendrycks, Dan and Burns, Collin and Basart, Steven and Zou, Andy and Mazeika, Mantas and Song, Dawn and Steinhardt, Jacob},
  booktitle=ICLR,
year={2021}
}

@inproceedings{dao2022flashattention,
  title={Flash{A}ttention: {F}ast and Memory-Efficient Exact Attention with {IO}-Awareness},
  author={Dao, Tri and Fu, Daniel Y. and Ermon, Stefano and Rudra, Atri and R{\'e}, Christopher},
  booktitle=NIPS,
  year={2022}
}

@article{penedo2024fineweb,
  title={The fineweb datasets: {D}ecanting the web for the finest text data at scale},
  author={Penedo, Guilherme and Kydl{\'\i}{\v{c}}ek, Hynek and Lozhkov, Anton and Mitchell, Margaret and Raffel, Colin A and Von Werra, Leandro and Wolf, Thomas and others},
  journal=NIPS,
  year={2024}
}

@misc{Karpathy2022,
  author = {Andrej Karpathy},
  title = {\text{NanoGPT}},
  year = {2022},
  publisher = {GitHub},
  journal = {GitHub repository},
  howpublished = {\url{https://github.com/karpathy/nanoGPT}},
  commit = {393a43d9a5c22450bbf06e78da2cb6eeef084b717}
}

@article{hochreiter1997long,
  title={Long short-term memory},
  author={Hochreiter, Sepp and Schmidhuber, J{\"u}rgen},
  journal={Neural computation},
  year={1997},
  publisher={MIT press}
}

@article{sakaguchi2021winogrande,
  title={Winogrande: {A}n adversarial winograd schema challenge at scale},
  author={Sakaguchi, Keisuke and Bras, Ronan Le and Bhagavatula, Chandra and Choi, Yejin},
  journal={Communications of the ACM},
  year={2021},
  publisher={ACM New York, NY, USA}
}

@article{li2025htr,
  title={{HTR-VT: H}andwritten text recognition with vision transformer},
  author={Li, Yuting and Chen, Dexiong and Tang, Tinglong and Shen, Xi},
  journal={Pattern Recognition},
  year={2025},
  publisher={Elsevier}
}

@inproceedings{cascianelli2022lam,
  title={The {LAM} dataset: {A} novel benchmark for line-level handwritten text recognition},
  author={Cascianelli, Silvia and Pippi, Vittorio and Maarand, Martin and Cornia, Marcella and Baraldi, Lorenzo and Kermorvant, Christopher and Cucchiara, Rita},
  booktitle={ICPR},
  year={2022}
}

@article{radford2019language,
  title={Language models are unsupervised multitask learners},
  author={Radford, Alec and Wu, Jeffrey and Child, Rewon and Luan, David and Amodei, Dario and Sutskever, Ilya and others},
  journal={OpenAI blog},
  year={2019}
}

@article{team2024gemma,
  title={{Gemma 2: I}mproving open language models at a practical size},
  author={Team, Gemma and Riviere, Morgane and Pathak, Shreya and Sessa, Pier Giuseppe and Hardin, Cassidy and Bhupatiraju, Surya and Hussenot, L{\'e}onard and Mesnard, Thomas and Shahriari, Bobak and Ram{\'e}, Alexandre and others},
  journal={arXiv:2408.00118},
  year={2024}
}

@article{defossez2024moshi,
  title={{Moshi: A} speech-text foundation model for real-time dialogue},
  author={D{\'e}fossez, Alexandre and Mazar{\'e}, Laurent and Orsini, Manu and Royer, Am{\'e}lie and P{\'e}rez, Patrick and J{\'e}gou, Herv{\'e} and Grave, Edouard and Zeghidour, Neil},
  journal={arXiv:2410.00037},
  year={2024}
}

@article{zhang2023google,
  title={{Google USM: S}caling automatic speech recognition beyond 100 languages},
  author={Zhang, Yu and Han, Wei and Qin, James and Wang, Yongqiang and Bapna, Ankur and Chen, Zhehuai and Chen, Nanxin and Li, Bo and Axelrod, Vera and Wang, Gary and others},
  journal={arXiv:2303.01037},
  year={2023}
}

@article{baevski2020wav2vec,
  title={{wav2vec 2.0: A} framework for self-supervised learning of speech representations},
  author={Baevski, Alexei and Zhou, Yuhao and Mohamed, Abdelrahman and Auli, Michael},
  journal=NIPS,
  year={2020}
}

@inproceedings{chiu2022bestrq,
  title={Self-supervised learning with random-projection quantizer for speech recognition},
  author={Chiu, Chung-Cheng and Qin, James and Zhang, Yu and Yu, Jiahui and Wu, Yonghui},
  booktitle={International Conference on Machine Learning},
  year={2022},
  organization={PMLR}
}

@article{zellers2019hellaswag,
  title={Hella{S}wag: {C}an a machine really finish your sentence?},
  author={Zellers, Rowan and Holtzman, Ari and Bisk, Yonatan and Farhadi, Ali and Choi, Yejin},
  journal={arXiv:1905.07830},
  year={2019}
}
    }
    }
}
\clearpage
\maketitlesupplementary
\appendix

In this appendix, we report additional ablations  (\cref{sec:sup:ablation}) and analysis 
(\cref{sec:sup:analysis}), and a proof of our main theoretical results (\cref{sec:sup:proof}). We also propose an extended related work in depth (\cref{sec:sup:related}) and additional illustrations for the image generation experiment (\cref{sec:sup:quali}.
)

%============================
\section{Additional Ablations.}
\label{sec:sup:ablation}
%=============================

\paragraph{Higher Orders k.}
We study the effect of increasing the polynomial degree $k$ in \cref{tab:additonal}. 
Higher-order polynomials provide a modest improvement, at the cost of a slight speed reduction. 
The limited impact is expected, as the sequences evaluated in classic benchmarks are relatively short and do not fully stress the model’s representational capacity. 
As suggested by our theoretical analysis, higher-order interactions could become more critical for distinguishing longer sequences.

\paragraph{Additional Baselines.}
We evaluate competing linear attention models and report their performance in \cref{tab:additonal}.
\begin{itemize}
    \item \emph{Performer.}
We attempted to train a GPT2-S model using Performer; however, both \texttt{pytorch-performer} and the Performer implementation from \texttt{torch\_geometric} systematically diverged, producing NaN losses after a few iterations. 
At present, we are not aware of a well-maintained Performer implementation suitable for large-scale language model training, which prevented a fair comparison.
\item \emph{Mamba.}
We trained a Mamba model of comparable size (124M parameters) on the NLP task; results are reported in \cref{tab:additonal}. 
As expected, Mamba achieves excellent speed due to its $\mathcal{O}(1)$ complexity and dedicated CUDA kernel. 
PoM shares the same $\mathcal{O}(1)$ theoretical complexity, and we therefore expect that a dedicated CUDA implementation would yield comparable speedups.
In terms of accuracy, however, \textbf{Mamba performs below MHA and PoM} on the evaluated benchmarks. 
We used the same classic training recipe inherited from attention-based models for all methods, which suggests that Mamba may require more tuning adaptation than PoM. 
We also observed that \textbf{Mamba consumed approximately $4\times$ more memory} during training, forcing smaller batch sizes and gradient accumulation, whereas PoM and MHA have a similar memory footprint.
\end{itemize}

%============================
\section{Additional Analysis.}
\label{sec:sup:analysis}
%=============================

\paragraph{Comparison to FlashAttention.}
FlashAttention is a highly engineered implementation of attention, relying on custom CUDA kernels and carefully optimized memory management. 
In contrast, PoM is currently implemented entirely in high-level PyTorch, and yet already achieves consistent 2--4$\times$ speedups across tasks, and in some settings (notably NLP, see \cref{tab:additonal}, and image generation), PoM is faster than Flash-Attention. 
We therefore view these results as a conservative estimate of PoM’s speed, and expect further gains with dedicated low-level optimization.

\paragraph{Generalization Under Train--Test Length Mismatch.}
Since $H(x)$ is obtained by pooling over tokens, its dimensionality is independent of the sequence length. 
In our NLP experiments, PoM is trained with a fixed sequence length of 2048 tokens (with padding), while evaluation on HellaSwag involves sequences ranging from 48 to 742 tokens, with no observable impact on performance.

\begin{table}[t]
    \caption{{\bf Additional ablations and baselines.} MHA throughput uses Flash-Attention, Performer uses \texttt{torch\_geometric}.}
    \label{tab:additonal}
    \centering
    \resizebox{\linewidth}{!}{
    \begin{tabular}{lcccc}
    \toprule
     &
     {\makecell[c]{\textbf{ARC-easy} \\ \textbf{Acc.Norm}$\uparrow$}} 
     &
     {\makecell[c]{\textbf{HellaSwag} \\ \textbf{Acc.Norm}$\uparrow$}} 
     &
      {\makecell[c]{\textbf{Winogrande} \\ \textbf{Acc.}$\uparrow$}}
      &
      {\makecell[c]{\textbf{Throughput} \\ \textbf{\@ 4096 tok., tok/s}$\uparrow$}}
      \\\midrule
      PoM-hyb k=2 & 29.0          & 33.8          & \textbf{51.9} & 163 \\
      PoM-hyb k=3 & \textbf{29.5} & 33.4          & 49.6 & 161 \\
      PoM-hyb k=5 & 28.6          & \textbf{33.9} & 49.5 & 153 \\[-0.8mm]\greyrule
      MHA         & 29.4          & 33.3          & 49.4 &  42   \\
      Performer   & \multicolumn{3}{c}{training failed} & 59 \\
      Mamba       & 29.3          &  30.8         & 48.0 & \textbf{384} \\
      \bottomrule
    \end{tabular}
    }
    % \vspace{-5mm}
\end{table}

\section{Proof of Lemma 3}
\label{sec:sup:proof}
    We first need to show that set with different entries are mapped to different vectors. We first separate $\PoM$ into its two components:
    \begin{align}
        s(X) = \sigma(W_s X)\\
        H(X) &= \left[\sum_{p=1}^k \bm{\alpha}_p \odot h(W_h X)^p\right]\bm{1}\\
        \PoM(X) &= W_o \left[\sigma(W_s X) \odot H(X)\right]
    \end{align}

    Assuming $\ker(W_o) = \emptyset$, and noting that $H_k(X)$ is the same for every column, we just have to show that $s(X)$ has different columns. This is easily achieved by having $\ker(W_s) = \emptyset$ since $\sigma$ is injective and the composition of injective functions is itself injective.

    Second, we have to show that sets that differ by at least one element are mapped to all different entries. To simplify notations, we will consider the special case where all matrices are the identity or an identity block positioned such as to perform submatrix selection. All the matrices can thus be removed from the formula. A similar argument can be made for matrices that are full rank as they preserve injectivity. We will also consider linear activations everywhere, which can be made as close as one wish by partitioning the image of the activation function and performing piecewise linear approximation.

    With this simplified version of $\PoM$, we have to show that for 2 sets $X,X'$ differing by at least one element (i.e., $\exists x' \in X', \forall x\in X, x \neq x'$), then there exist $k$ such that
    \begin{align}
        \forall x\in X, x'\in X', x\sum_{x_i\in X} x_i^k \neq x\sum_{x_i\in X} x_i^k.
    \end{align}

    Consider the functions $P(t)$ and $P'(t)$ defined as follows:
    \begin{align}
        P(t)=\sum_{x_i in X} x_i^t\\
        P'(t)=\sum_{x_i in X'} x_i^t
    \end{align}

    Since $X$ and $X'$ differ by at least one element, there exists at least one $x_i \in X$ such that $x_i \neq x_i', \forall x_i' \in X'$. This implies that the functions $P(t)$ and $P'(t)$ are not identical since are sums of exponentials with different bases.

    Since $P(t)$ and $P'(t)$ are different functions, there must exist some $k$ for which $P(k) \neq P'(k)$. In other words, there exists a $k$ such that:
    \begin{align}
        \sum_{x_i\in X} x_i^k \neq \sum_{x_i' \in X'} x_i'^k
    \end{align}
    
    For this $k$, let us denote $S_k = \sum_{x_i\in X}x_i^k$. We need to show that $xS_k \neq x'S_k'$ for all $x\in X$ and $x'\in X'$.
    Assume for the sake of contradiction that there exist $x\in X$ and $x'\in X'$ such that $xS_k = x'S_k'$. This implies:
    \begin{align}
        x \sum_{x_i \in X} x_i^k = x' \sum_{x_i' \in X'} x_i'^k
    \end{align}

    Rearranging, we get:
    \begin{align}
        \frac{x}{x'} =  \frac{\sum_{x_i' \in X'}x_i'^k}{\sum_{x_i\in X}x_i^k}
    \end{align}

    Since $S_k\neq S_k'$, the right-hand side is not equal to 1. However, for this equality to hold for all $x\in X$ and $x'\in X'$, the ratio $x/x'$ would need to be constant for all pairs $(x,x')$, which is not possible given that $X$ and $X'$ differ by at least one element.

    Therefore, there exists a $k$ such that $x S_k\neq x'S_k'$ for all $x\in X$ and $x'\in X'$.

\section{Related work on Diffusion}
\label{sec:sup:related}
\paragraph{Diffusion}
Diffusion models~\cite{ho20nips,nichol21icml,song21iclr} learn a neural operator that produces natural images from noise using a forward-reverse set of processes.
The forward process consists in pushing the distribution of natural images forward to a known distribution, typically Gaussian, which can be done by adding increasing level of noise to the image.
The reverse process does not have an explicit solution, but can be approximated by a neural network by regressing the local inverse of the forward process, \textit{i.e.}, solving
\begin{align}
    &\min_\theta \mathbb{E}_{t\sim \mathcal{U}(0,1)}\left[\|\varepsilon_t - f_\theta(x_t, t)\|^2\right],\\
    &\text{ s.t. } x_t = \alpha_t x_0 + \gamma_t \varepsilon_t,\, \varepsilon_t \sim \mathcal{N}(0,1).
\end{align}
Here, $\alpha_t$ and $\gamma_t$ are chosen such that $x_0$ corresponds to a natural image whereas $x_1$ corresponds to pure Gaussian noise.
A great amount of research has been put into finding better noise schedules ($\alpha_t$ and $\gamma_t$)~\cite{balaji22eDiffI,karras24cvpr,hang2024improved}, or improving the quantity that is regressed~\cite{lipman22iclr,shi24nips,liu23iclr}, keeping the general idea of learning to invert step by step the stochastic differential equation that transforms an image into noise.

For image and generation, most efforts have been poured into designing efficient architectures at the task. While the original DDPM papers~\cite{ho20nips,nichol21icml} sample images in pixel space, making it unsuitable for large resolution, the most groundbreaking improvement was introduced by Stable Diffusion~\cite{rombach22cvpr} with the addition of a variational auto-encoder (VAE) that allows the diffusion process to be performed in a lower dimensional latent space.
Stable Diffusion uses a U-Net architecture complemented by attention layers~\cite{rombach22cvpr,Si_2024_CVPR}.
To benefit more from the scaling properties of transformers~\cite{kaplan2020scaling,zhai22cvpr}, simpler approaches based solely on transformer layers has been proposed in DiT~\cite{peebles23iccv} and the subsequent flow-matching version SiT~\cite{ma24eccv}.
Most modern text-to-image generation models are now based on Transformer layers rather than the U-Net~\cite{hatamizadeh2025diffit,esser24icml,chen2024pixart,gao2024lumina}. \cite{crowson2024scalable, gu2023matryoshka}, train efficient pixel space transformers models by leveraging multiscale training and SwinAttention.
Similarly, RIN~\cite{jabri23icml, chen2023fit} also proposes an approach using attention only, albeit in a Perceiver-IO~\cite{jaegle22iclr} inspired architecture that uses cross-attention to perform most of the computation in a smaller latent space, and has been successfully extended to text-to-image~\cite{dufour24cvpr}.
In addition to architectures and sampling~\cite{Zhou_2024_CVPR,Bai_2024_CVPR,zhao2023mobilediffusion}, the importance of training is also highlighted in recent works, from resampling the training data~\cite{Gokaslan_2024_CVPR,Liu_2024_CVPR} to RL~\cite{Wallace_2024_CVPR,wei2024powerful,lee2025parrot} and model averaging~\cite{Karras_2024_CVPR}.

\section{Uncurated Image Samples}
\label{sec:sup:quali}
To show \pomicon PoM versatility, with train a DiPoM-XL/2 with the diffusion loss instead of the flow-matching loss and show generated samples with CFG $\omega=6$ in the following pages.

\newcommand{\addimages}[1]{%
% Read the image list from the file
\readarraysepchar{,}
\readdef{#1/images.txt}\imageList
\readarray*\imageList\imageArray[-,1]
\begin{tikzpicture}
    % Define the scale factor to fit the figure within the page width
    \pgfmathsetmacro{\scaleFactor}{\textwidth / (5 * 256)}

    % Loop over the first 60 images
    \foreach \i in {1,...,30} {
        \pgfmathsetmacro{\row}{int(floor((\i - 1) / 5))} % Adjust for 1-based indexing
        \pgfmathsetmacro{\col}{int(mod(\i - 1, 5))} % Adjust for 1-based indexing

        % Get the image filename from the array
        \edef\image{\imageArray[\i,1]}
        \pgfmathsetmacro{\xcoord}{\col * 256 * \scaleFactor}
        \pgfmathsetmacro{\ycoord}{-\row * 256 * \scaleFactor}

        % Print the coordinates for debugging
        \typeout{Coordinates: (\xcoord, \ycoord)}
        \pgfmathsetmacro{\imgsize}{256 * \scaleFactor}

        \node[inner sep=0pt] at (\xcoord pt, \ycoord pt) {
            \includegraphics[width=\imgsize pt, height=\imgsize pt]{#1/\image}
        };
    }
\end{tikzpicture}
}

\begin{figure*}[p]
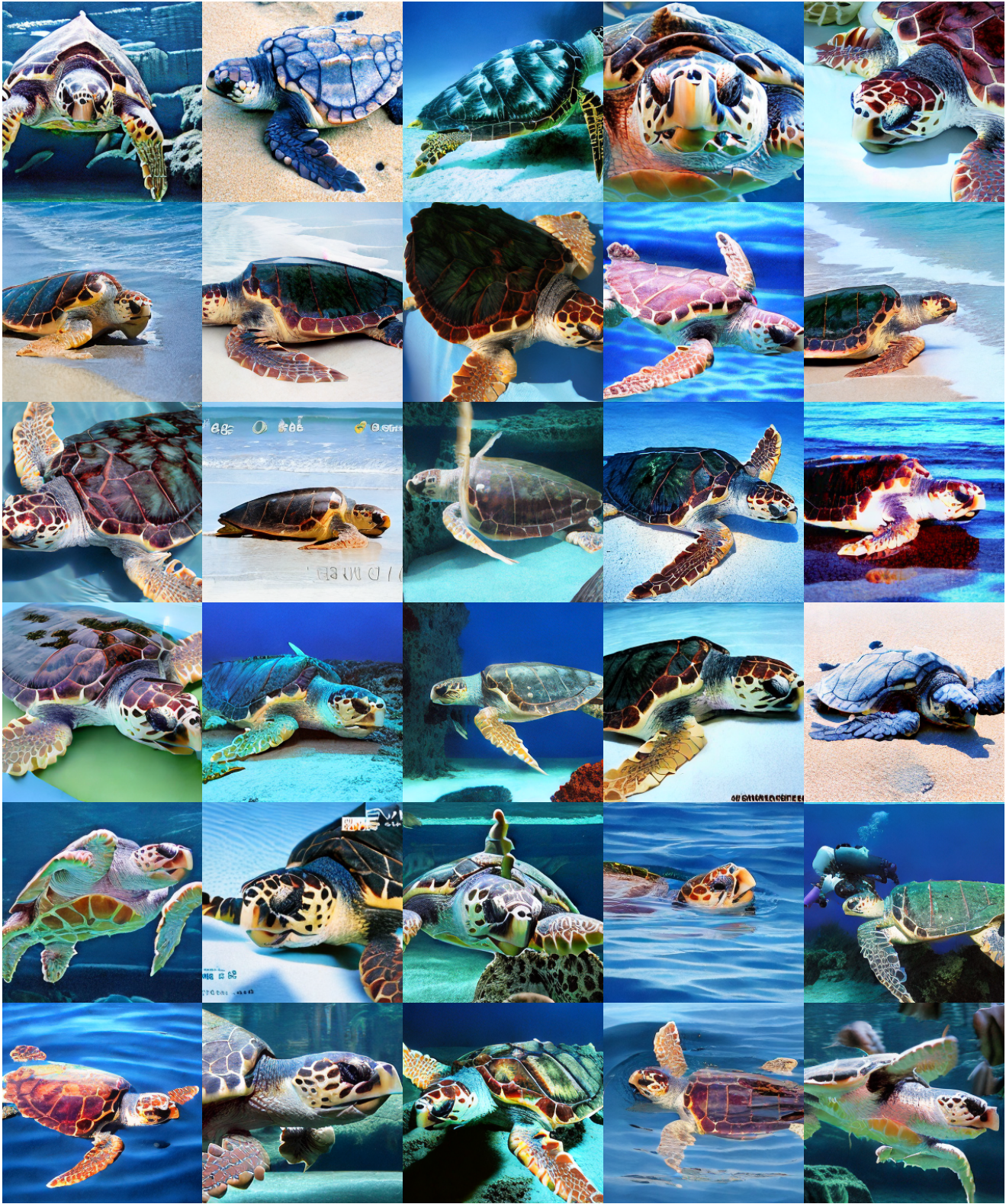

    \begin{center}
    \addimages{images/33/}
    \end{center}
    \caption{Uncurated 256² images for the class \emph{loggerhead, loggerhead turtle, Caretta caretta} (33).}
    \label{fig:turtle}
\end{figure*}

\begin{figure*}[p]
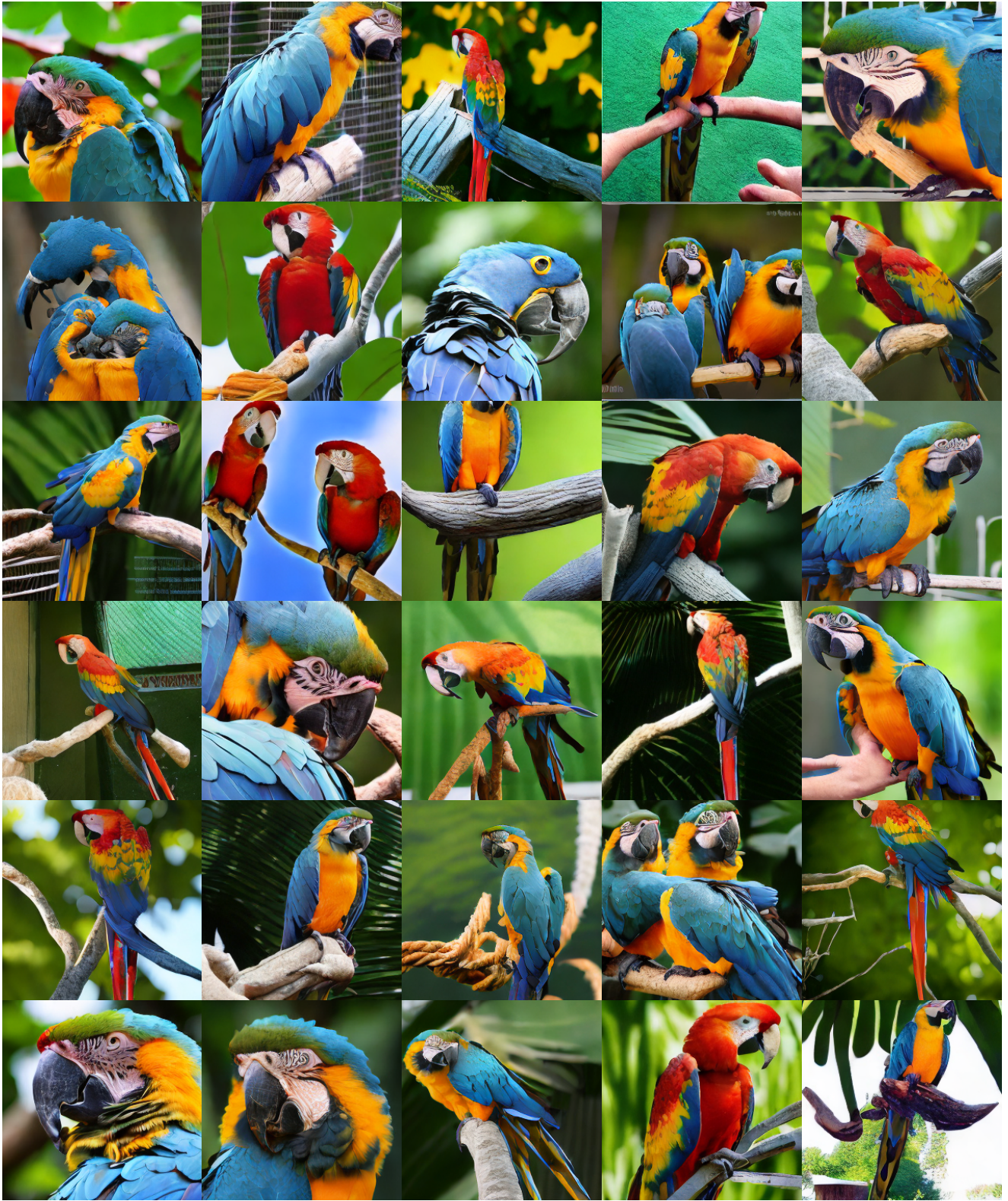

    \begin{center}
    \addimages{images/88/}
    \end{center}
    \caption{Uncurated 256² images for the class \emph{macaw} (88).}
    \label{fig:macaw}
\end{figure*}

\begin{figure*}[p]
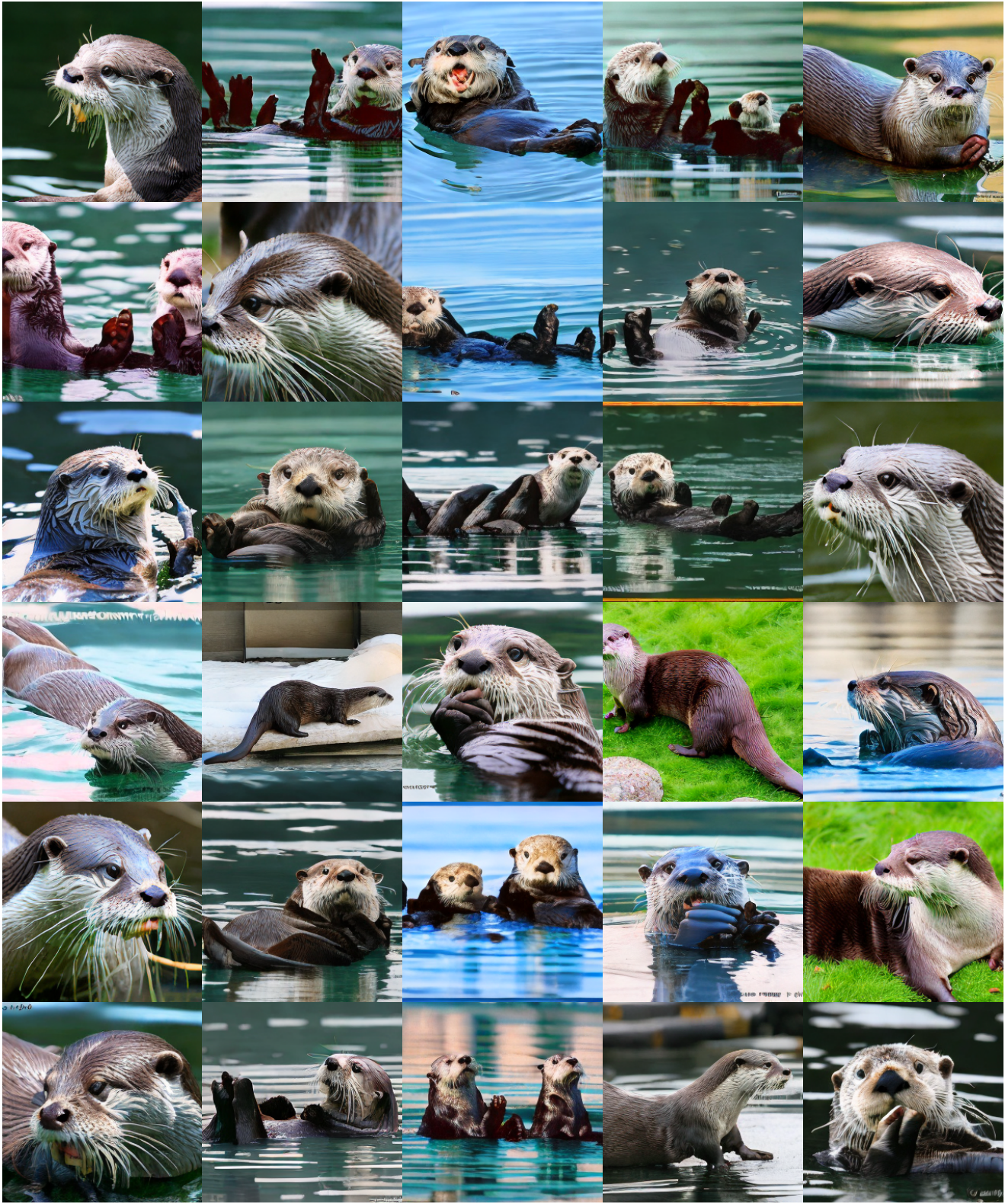

    \begin{center}
    \addimages{images/360/}
    \end{center}
    \caption{Uncurated 256² images for the class \emph{otter} (360).}
    \label{fig:otter}
\end{figure*}

\begin{figure*}[p]
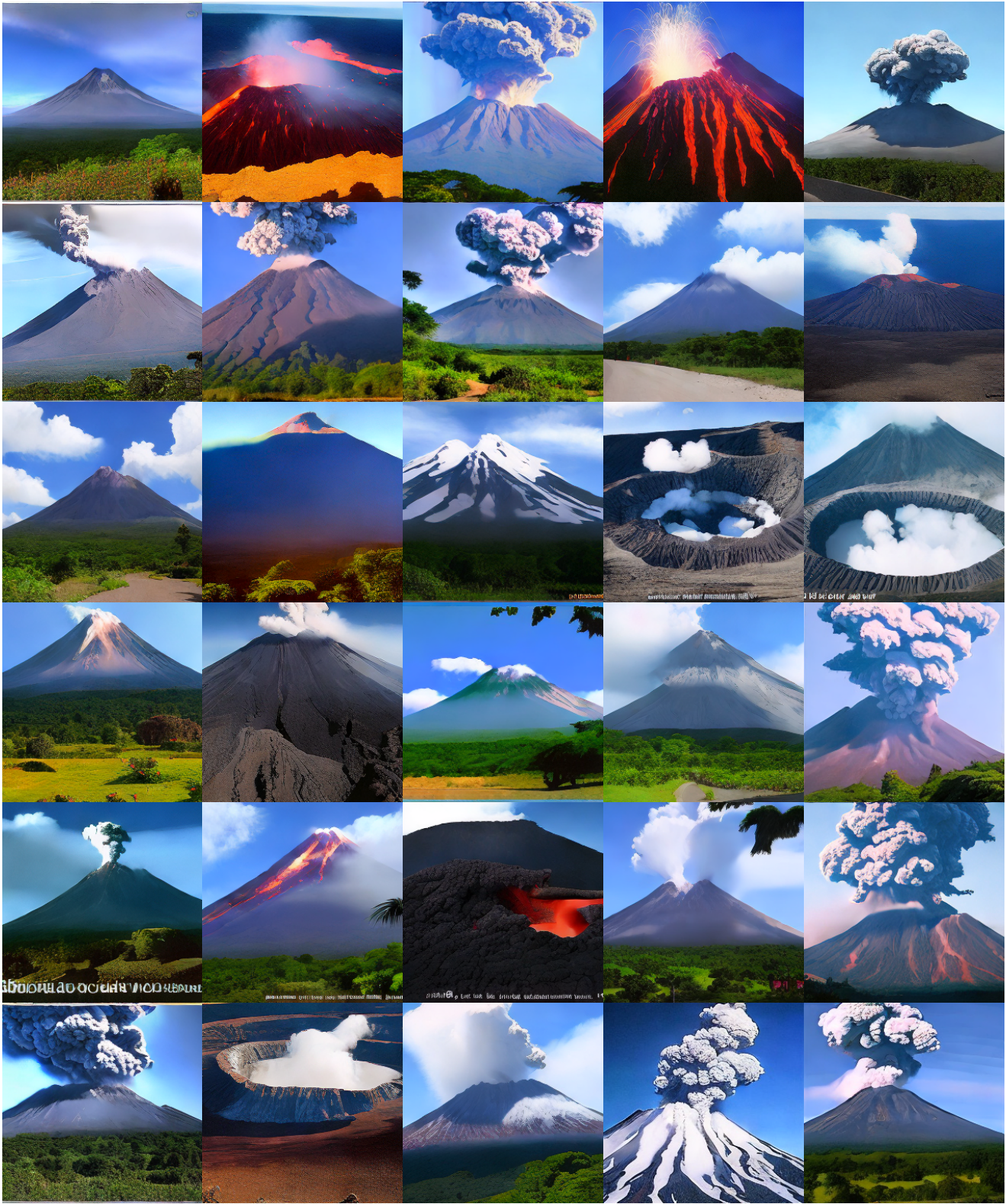

    \begin{center}
    \addimages{images/980/}
    \end{center}
    \caption{Uncurated 256² images for the class \emph{volcano} (980).}
    \label{fig:volcano}
\end{figure*}

\end{document}